\documentclass{article} 
\usepackage{qatfm_preprint,times}

\usepackage{wrapfig} 
\usepackage[utf8]{inputenc} 
\usepackage[T1]{fontenc}    
\usepackage{hyperref}       
\usepackage{url}            
\hypersetup{hidelinks}
\usepackage{booktabs}       
\usepackage{amsfonts}       
\usepackage{nicefrac}       
\usepackage{xpatch}
\usepackage{natbib}  
\usepackage{xpatch}
\usepackage{graphicx}
\usepackage{subfigure}
\usepackage{adjustbox}
\usepackage{enumitem} 

\usepackage{bm}
\usepackage{colortbl} 
\usepackage{amsmath}
\usepackage{amssymb} 

            \def\x{\bm{x}}

\def\A{\mathbf{A}}    \def\B{\mathbf{B}}

\def\RR{\mathbb{R}}
\def\EE{\mathbb{E}}

\usepackage{algorithm}
\usepackage{algpseudocode}
\usepackage{subcaption}

\newtheorem{theorem}{Theorem}
\newtheorem{proposition}{Proposition}

\newtheorem{remark}{Remark}

\newtheorem{definition}{Definition}
\newenvironment{proof}[1][Proof]{\par\noindent\textbf{#1.} }{\hfill$\square$\par}
\usepackage{multirow}
\usepackage{pifont}
\usepackage{amssymb}
\newcommand{\cmark}{\textcolor{green!60!black}{\ding{51}}}
\newcommand{\xmark}{\textcolor{red}{\ding{55}}}
\usepackage{threeparttable}
\usepackage{colortbl}
\usepackage{threeparttable}

\title{Beyond Straightness: Non-Crossing Flow Matching via Quantile AlignTree Coupling}

\author{
  Junyi Lin\textsuperscript{1}, 
  Mengyu Li \textsuperscript{2}, 
  Jingxuan Hu\textsuperscript{1}, 
  Kejun He \thanks{Corresponding author}\textsuperscript{1}, 
  Cheng Meng \thanks{Corresponding author}\textsuperscript{1} \\
  \textsuperscript{1}Institute of Statistics and Big Data, Renmin University of China, Beijing, China \\
  \textsuperscript{2}Department of Statistics and Data Science, Tsinghua University, Beijing, China \\
  \texttt{\{junyilin, 2025104242, chengmeng, kejunhe\}@ruc.edu.cn}, \\
  \texttt{mengyuli@tsinghua.edu.cn}
}

\preprintfinalcopy 
\begin{document}

\maketitle

\begin{abstract}
The performance of Flow Matching largely depends on the quality of the coupling between the source and target distributions. However, independent coupling often leads to path crossings and local velocity ambiguity, while OT-based couplings typically incur high construction costs. To address this challenge, we propose \textbf{Quantile AlignTree Flow Matching (QAT-FM)}, an efficient structured coupling strategy that constructs a hierarchical coupling between a Gaussian prior and the target data distribution via a quantile-aligned tree structure. QAT-FM constructs the coupling in $\mathcal{O}(Nd\log N)$ time and supports per-pair source sampling with $\mathcal{O}(d)$ complexity, enabling scalable training for large-scale high-dimensional generative tasks. Theoretically, we prove that the QAT coupling satisfies marginal consistency, induces non-crossing linear interpolation paths, and consistently improves path separation at intermediate times compared with independent coupling, thereby alleviating local velocity ambiguity. QAT-FM further extends naturally to conditional generation, enabling structured conditional coupling while preserving global Gaussian alignment. Experiments across diverse benchmark datasets demonstrate that QAT-FM achieves competitive generative performance while substantially reducing coupling construction cost.
\end{abstract}

\section{Introduction}
\label{sec:intro}

In recent years, continuous-time flow-based generative models have become increasingly prominent in generative modeling~\citep{chen2018neural, grathwohl2019ffjord, lipman2022flow, liu2023flow}. 
A representative class is continuous normalizing flows (CNFs), which define 
continuous transformations as the solution flows of ordinary differential equations (ODEs) whose time-dependent velocity fields are parameterized by neural networks.
During sampling, new data are generated by numerically integrating the learned ODE dynamics.
To improve training efficiency, Flow Matching~\citep{lipman2022flow} directly supervises the velocity field with targets derived from pre-specified conditional probability paths, avoiding the repeated numerical ODE integration required by conventional CNF training, enabling scalable training of generative models. 
Flow Matching has achieved competitive performance across diverse domains, including image generation~\citep{ma2024sit, esser2024scaling}, speech synthesis~\citep{le2023voicebox, liu2024generative}, and text generation~\citep{hu2024flowseq, hu2026elf}.

Although Flow Matching improves the scalability of training continuous-time generative models, its performance remains sensitive to the conditional probability paths and source--target couplings constructed during training~\citep{lipman2022flow, tong2024improving, pooladian2023multisample}. In common linear-path formulations, Flow Matching trains the velocity field by regressing target velocities along linear interpolation paths between noise and data samples; different coupling choices therefore directly alter the structure of the supervision signal. 
When training paths cross, or pass too closely through the same intermediate region with substantially different target velocities, \emph{velocity ambiguity} arises locally; under mean-squared-error training, these inconsistent velocity directions are averaged, increasing the difficulty of fitting the velocity field~\citep{guo2025variational}. 
By contrast, training paths with the \emph{non-crossing} property or good \emph{path separation} reduce this local ambiguity, yielding a more deterministic mapping from intermediate states to target velocities and thereby improving training stability. 
In conditional generation, coupling strategies must also preserve the alignment between samples and conditions, since couplings that ignore the conditioning information may produce geometrically favorable but conditionally inconsistent supervision signals~\citep{chemseddine2025conditional, cheng2025curse}.

Several existing approaches seek to alleviate path crossing and velocity ambiguity by improving coupling quality. The reflow procedure in Rectified Flow~\citep{liu2023flow} reconstructs pairings between noise and data using a learned model, yielding more deterministic and straighter generation trajectories; however, this requires additional sampling and re-training steps, imposing substantial computational overhead. OT-based methods~\citep{villani2009optimal} construct geometrically optimal pairings by minimizing transport cost; under a quadratic cost, the induced displacement interpolation exhibits a favorable non-crossing structure that alleviates velocity ambiguity. Nonetheless, exact empirical OT is computationally expensive and difficult to scale to large datasets; in conditional generation, maintaining conditioning consistency typically requires condition-aware constraints or condition-wise couplings, further increasing algorithmic complexity. Mini-batch OT~\citep{pooladian2023multisample, tong2024improving} substantially reduces this cost by solving OT within each mini-batch, but the resulting coupling is primarily local and cannot strictly preserve the non-crossing structure at the full-dataset level. Most of the above methods construct couplings only over finite Gaussian samples; SD-FM~\citep{mousavi2026flow} further improves the alignment of continuous Gaussian information via gradient-based optimization, but its iterative optimization incurs substantial sampling or optimization cost and does not yield an explicit low-cost global coupling construction.

\begin{table}[ht]
\centering
\footnotesize
\caption{\footnotesize 
Comparison of Flow Matching coupling strategies.}
\label{tab:flow_comparison}
\begin{threeparttable}
\begin{tabular}{lccccc}
\toprule
\textbf{Method}
& \textbf{Build Cost}\tnote{1}
& \textbf{Sample Cost}
& \textbf{Non-Cross}
& \textbf{Path Sep.}
& \textbf{Gaussian Align.} \\
\midrule
FM
& --- & $\mathcal{O}(Nd)$ & \xmark & \xmark & \cmark\cmark \\

RecFlow
& $\mathcal{O}(N E \Theta)$ & $\mathcal{O}(N \Theta)$ & \cmark & \xmark & \xmark \\

OT-FM$_{\text{(Full)}}$
& $\mathcal{O}(N^3+N^2d)$ & $\mathcal{O}(N^2+Nd)$ & \cmark & \cmark & \xmark \\

OT-FM$_{\text{(Minibatch)}}$
& $\mathcal{O}(NB+NBd)$ & $\mathcal{O}(NB+Nd)$ & \xmark & \xmark & \xmark \\

SD-FM
& $\mathcal{O}(NKd)$ & $\mathcal{O}(N^2d)$ & \cmark & \cmark & \cmark \\

\rowcolor{gray!15}
\textbf{QAT-FM$_{\textbf{(Ours)}}$}
& $\mathcal{O}(Nd\log N)$ 
& $\mathcal{O}(Nd)$ 
& \cmark 
& \cmark\cmark\tnote{2} 
& \cmark\cmark \\
\bottomrule
\end{tabular}
\begin{tablenotes}
\footnotesize
\item[1] \textbf{Notation.} $N$: dataset size; $B$: mini-batch size; $d$: data dimension; $K$: SD optimization steps; $E$: reflow rounds; $\Theta$: one network forward pass. Complexities are reported up to constant factors.
\item[2] \cmark\cmark indicates that the property is also preserved in conditional generation.
\end{tablenotes}
\end{threeparttable}
\end{table}

Table~\ref{tab:flow_comparison} compares representative Flow Matching coupling strategies along five axes: coupling construction cost (\textbf{Build Cost}), sampling cost (\textbf{Sample Cost}), whether the coupling induces non-crossing paths at the full-dataset level (\textbf{Non-Cross}), whether it improves path separation (\textbf{Path Sep.}), and whether the source marginal is aligned with a continuous Gaussian prior (\textbf{Gaussian Align.}). As shown, prior methods generally trade off computational efficiency, path geometry, and Gaussian-prior alignment. To address these limitations, we propose \textbf{Quantile AlignTree Flow Matching (QAT-FM)}, a low-cost structured coupling strategy that achieves non-crossing paths, improved path separation, and continuous Gaussian-prior alignment, with several of these properties naturally extending to conditional generation.

\noindent\textbf{Contributions.} Our main contributions are as follows:
\begin{itemize}[noitemsep, topsep=0pt, left=5pt]
    \item[1)] We propose \textbf{Quantile AlignTree (QAT)}, an efficient semi-discrete transport construction based on binary tree-structured quantile alignment. QAT induces an exact coupling between an empirical data distribution and a continuous Gaussian prior. With construction complexity $\mathcal{O}(Nd\log N)$ and per-pair sampling complexity $\mathcal{O}(d)$, QAT scales to large-scale, high-dimensional generative modeling tasks.

    \item[2)] The QAT coupling supports multiple tree-splitting strategies and can naturally incorporate conditioning information into both tree construction and coupling.

    \item[3)] We prove that QAT induces favorable path structure: the resulting linear interpolation paths satisfy the \emph{non-crossing} property, and path separation at intermediate times is uniformly greater than that of standard independent-coupling flow matching, alleviating local velocity ambiguity.

    \item[4)] We conduct systematic experiments on low-dimensional synthetic data, unconditional generation, discrete class-conditional generation, continuous text-conditional generation, and large-scale image synthesis. Results show that QAT-FM achieves competitive generation quality with substantially lower coupling construction complexity than OT-based methods.
\end{itemize}

\section{Preliminaries}
\label{sec:pre}

\paragraph{Notation.}

Let $\mathcal{P}(\RR^d)$ denote the set of Borel probability measures on $\RR^d$. For $\alpha\in\mathcal{P}(\RR^d)$ and $\beta\in\mathcal{P}(\RR^{d'})$, a joint distribution $\pi\in\mathcal{P}(\RR^d\times\RR^{d'})$ is called a \emph{coupling} of $\alpha$ and $\beta$ if its marginals are $\alpha$ and $\beta$; we write $\Pi(\alpha,\beta)$ for the collection of all such couplings. For a vector $\x\in\RR^d$, let $\x^{(j)}$ denote its $j$-th coordinate. For $\rho\in\mathcal{P}(\RR^d)$, let $\rho^{(j)}$ denote its $j$-th marginal distribution, $F_\rho^{(j)}$ the cumulative distribution function (CDF) of $\rho^{(j)}$, and $\bigl(F_\rho^{(j)}\bigr)^{-1}$ its generalized inverse. For a measurable map $T:\RR^d\to\RR^k$, we write $T_\#\rho$ for the pushforward of $\rho$ under $T$.

\paragraph{Flow-based Generative Modeling.}

Given a time-dependent velocity field $v:[0,1]\times\RR^d\to\RR^d$, write $v_t(x):=v(t,x)$. The \emph{flow map} $\phi_t:\RR^d\to\RR^d$ induced by $v_t$ is defined as the solution to the ODE
\[
\frac{d}{dt}\phi_t(x)=v_t(\phi_t(x)),\qquad \phi_0(x)=x.
\]
Setting $\rho_t=(\phi_t)_\#\gamma$, where $\gamma$ denotes the source distribution, the pair $(\rho_t,v_t)$ satisfies the continuity equation
\begin{equation}
\label{eq:ceq}
    \partial_t\rho_t+\nabla\cdot(\rho_t v_t)=0,
\end{equation}
with $\rho_0=\gamma$. When $(\phi_1)_\#\gamma=\mu$, the flow transports the source $\gamma$ to the target distribution $\mu$. The goal of flow-based generative modeling is therefore to learn a parameterized velocity field $v_\theta$ whose induced flow map satisfies $(\phi_1)_\#\gamma=\mu$.

Flow Matching~\citep{lipman2022flow} provides a tractable training procedure for this velocity field. It specifies a probability path $(\rho_t)_{t\in[0,1]}$ connecting $\gamma$ to $\mu$ together with a compatible target velocity field $u_t$. The parameterized velocity field is then learned via the regression objective
\begin{equation}
\label{eq:FM}
    \mathcal{L}_{\mathrm{FM}}(\theta)
    =
    \EE_{\substack{
    t\sim\mathrm{Unif}[0,1]\\
    x_t\sim\rho_t
    }}
    \left\|
    v_\theta(t,x_t)-u_t(x_t)
    \right\|^2.
\end{equation}
This formulation avoids backpropagating through numerical ODE solvers during training.

\paragraph{Conditional Flow Matching.}

In practice, the marginal velocity field $u_t$ is often intractable to construct directly. Conditional Flow Matching\footnote{Here, ``conditional'' refers to conditioning on paired endpoints. To avoid ambiguity with conditional generation, we refer to this formulation simply as Flow Matching (FM) throughout the remainder of the paper unless otherwise specified.} (CFM)~\citep{lipman2022flow, tong2024improving} addresses this by conditioning on paired endpoints: given $\pi\in\Pi(\gamma,\mu)$, draw $(x_0,x_1)\sim\pi$, specify a conditional path $p_t(\cdot\mid x_0,x_1)$ and its compatible velocity $u_t(\cdot\mid x_0,x_1)$, and train $v_\theta(t,x_t)$ to match $u_t(x_t\mid x_0,x_1)$, where $x_t\sim p_t(\cdot\mid x_0,x_1)$. A canonical and widely used choice is the linear interpolation path
\[
x_t=(1-t)x_0+tx_1,
\qquad
u_t(x_t\mid x_0,x_1)=x_1-x_0,
\]
under which the training objective reduces to
\begin{equation}
\label{eq:linear-cfm}
\mathcal{L}_{\mathrm{CFM}}(\theta)
=
\EE_{\substack{
t\sim\mathrm{Unif}[0,1]\\
(x_0,x_1)\sim\pi
}}
\left\|
v_\theta\!\left(t,(1-t)x_0+tx_1\right)
-
(x_1-x_0)
\right\|^2.
\end{equation}
This reveals that the endpoint pairing is governed by the coupling $\pi\in\Pi(\gamma,\mu)$. The simplest choice is the independent coupling $\pi=\gamma\otimes\mu$. A more principled alternative is the \emph{OT coupling}
\begin{equation}
\label{eq:OT_def}
    \pi_{\mathrm{OT}}
    \in
    \arg\min_{\pi\in\Pi(\gamma,\mu)}
    \int_{\RR^d\times\RR^d}
    \|x_0-x_1\|^2\,d\pi(x_0,x_1),
\end{equation}
which under the linear path tends to produce shorter, geometrically natural paired trajectories~\citep{pooladian2023multisample, tong2024improving}. However, computing the exact global OT coupling is expensive at large scale; in practice, mini-batch OT is used as an approximation, reducing computational cost at the expense of coupling exactness and cross-batch consistency. For the special case $\gamma=\mathcal{N}(0,I_d)$, \cite{mousavi2026flow} proposes solving a semi-discrete OT problem in~\eqref{eq:OT_def} to better align the continuous Gaussian source with the empirical target distribution.

\paragraph{Data Partitioning by Trees.}
The preceding discussion shows that the choice of coupling $\pi$ directly governs the geometry of particle trajectories in FM, while exact global OT coupling is prohibitively expensive at large scale~\citep{zhang2025fitting, mousavi2026flow}. 
This motivates binary tree-based space partitioning, which provides a simple way to organize data hierarchically and will serve as the foundation of our coupling construction.
Given a data space $\mathcal{X}\subseteq\RR^d$ and a dataset $\mathcal{D}=\{x_i\}_{i=1}^N\subseteq\mathcal{X}$, let $\widehat{\mu}_{\mathcal{D}}=N^{-1}\sum_{i=1}^N\delta_{x_i}$
denote its empirical distribution. A \emph{binary partition tree} $T$ consists of an internal node set $\mathcal{I}$ and a leaf node set $\mathcal{L}$. Each node $\eta$ is associated with a region $R_\eta\subseteq\mathcal{X}$ and data subset $\mathcal{D}_\eta=\mathcal{D}\cap R_\eta$, where the root $r$ satisfies $R_r=\mathcal{X}$.

For each internal node $\eta\in\mathcal{I}$, let $\widehat{\mu}_\eta=|\mathcal{D}_\eta|^{-1}\sum_{x\in\mathcal{D}_\eta}\delta_x$ denote the conditional empirical measure on $\mathcal{D}_\eta$. A \emph{split strategy} selects a split coordinate $\kappa_\eta\in[d]$ and threshold $s_\eta\in\RR$ based on $\widehat{\mu}_\eta$, partitioning $R_\eta$ into
\[
R_\eta^-=\{x\in R_\eta:x^{(\kappa_\eta)}\le s_\eta\}, \qquad
R_\eta^+=\{x\in R_\eta:x^{(\kappa_\eta)}>s_\eta\},
\]
and correspondingly partitioning $\mathcal{D}_\eta$ into $\mathcal{D}_\eta^-=\mathcal{D}_\eta\cap R_\eta^-$ and $\mathcal{D}_\eta^+=\mathcal{D}_\eta\cap R_\eta^+$. After recursive splitting, the leaf regions form a partition of $\mathcal{X}$, denoted $\mathcal{R}(T):=\{R_\ell:\ell\in\mathcal{L}\}$.


\section{Quantile AlignTree Flow Matching}
\label{sec:method}

\subsection{Quantile AlignTree Coupling}
\label{subsec:qat}

The preceding analysis shows that in the linear FM objective, the coupling $\pi$ governs how source samples are paired with target samples, directly shaping the geometry of flow paths. The key idea of QAT is to build a binary partition tree over the empirical data distribution and transfer its topology to the source distribution by matching the split mass ratios at each internal node.
Figure~\ref{fig:pipeline} summarizes the resulting coupling and sampling procedure.

Let $\widehat{\mu}_{\mathcal{D}}$ denote the empirical measure of dataset $\mathcal{D}=\{x_i\}_{i=1}^N\subseteq\RR^d$, and let $\gamma\in\mathcal{P}(\RR^d)$ be the source distribution, e.g., $\gamma=\mathcal{N}(0,I_d)$. Following the binary partition tree construction of Section~\ref{sec:pre}, we build a data tree $T(\widehat{\mu}_{\mathcal{D}})$ on $\widehat{\mu}_{\mathcal{D}}$. For each internal node $\eta\in\mathcal{I}$, let $(\kappa_\eta,s_\eta)$ denote its split coordinate and threshold, and let $\widehat{\mu}_\eta$ be the conditional empirical measure at that node. Based on the data tree, we define the Quantile AlignTree as follows.

\begin{definition}[Quantile AlignTree (QAT)]
\label{def:qat}
Given the data tree $T(\widehat{\mu}_{\mathcal{D}})$ and source distribution $\gamma$, the \emph{Quantile AlignTree} of $\gamma$ with respect to the data tree, denoted $\widetilde{T}_\gamma(T)$, is constructed recursively with the same topology as $T(\widehat{\mu}_{\mathcal{D}})$. At each internal node $\eta$, the split coordinate $\kappa_\eta$ is inherited from the data tree, but the data-side threshold $s_\eta$ is replaced by the source-side aligned threshold
\begin{equation}
\label{eq:qat_split}
\widetilde{s}_\eta
=
\left(F_{\gamma_\eta}^{(\kappa_\eta)}\right)^{-1}
\!\left(
F_{\widehat{\mu}_\eta}^{(\kappa_\eta)}(s_\eta)
\right),
\end{equation}
where $\gamma_\eta$ denotes the conditional distribution of $\gamma$ on the current source-side region $\widetilde{R}_\eta$. The source-side region is then partitioned as
\[
\widetilde{R}_\eta^-
=
\{z\in\widetilde{R}_\eta:z^{(\kappa_\eta)}\le\widetilde{s}_\eta\},
\qquad
\widetilde{R}_\eta^+
=
\{z\in\widetilde{R}_\eta:z^{(\kappa_\eta)}>\widetilde{s}_\eta\},
\]
and the left and right child nodes are constructed recursively.
\end{definition}

The quantile-based aligned threshold~\eqref{eq:qat_split} guarantees that the source side and the data side carry equal mass ratios at every split. Specifically, when the source-side conditional marginal is continuous,
\[
\gamma_\eta(\widetilde{R}_\eta^-)
=
\widehat{\mu}_\eta(R_\eta^-),
\qquad
\gamma_\eta(\widetilde{R}_\eta^+)
=
\widehat{\mu}_\eta(R_\eta^+).
\]
Consequently, the data tree and the aligned source tree carry equal mass at corresponding leaf nodes, naturally inducing a coupling from $\gamma$ to $\widehat{\mu}_{\mathcal{D}}$.

\begin{figure}[t]
    \centering
    \includegraphics[width=\linewidth]{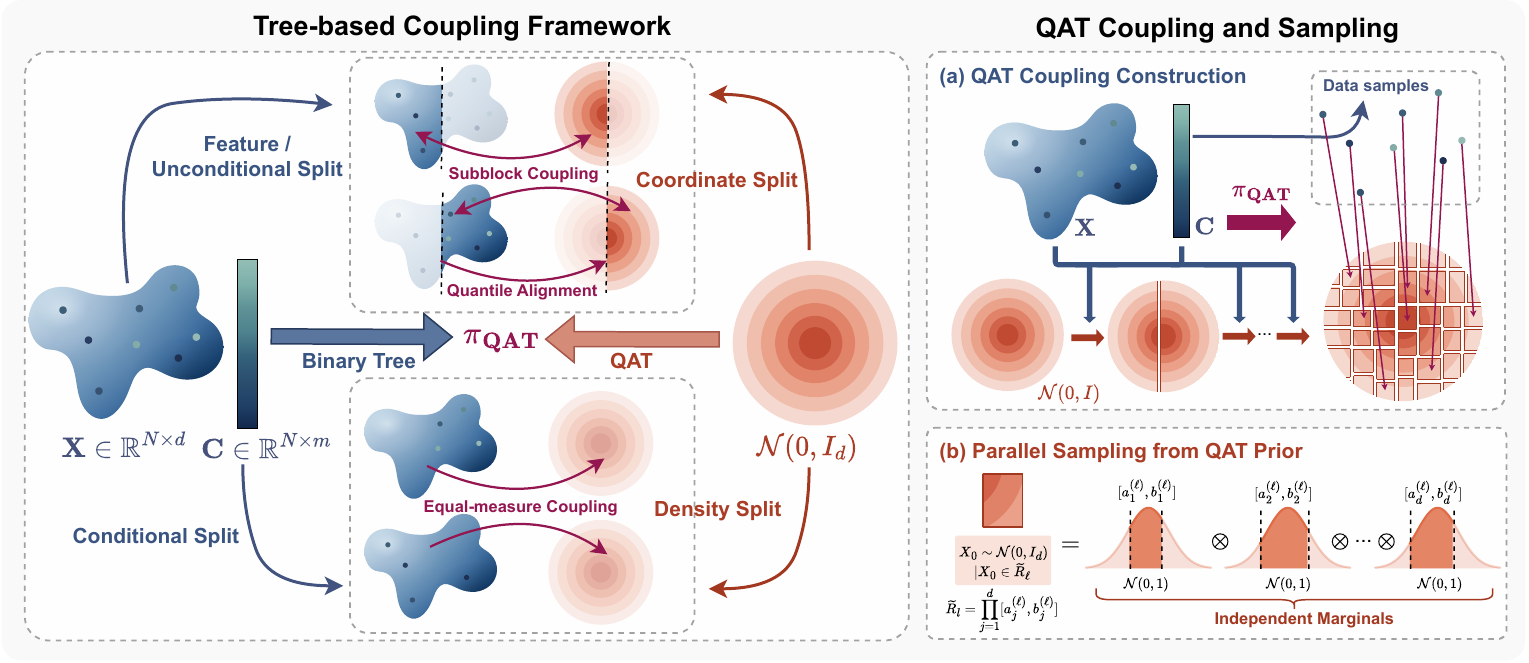}
    \caption{\footnotesize
    Overview of the QAT coupling and sampling procedure.
    (\emph{Left}) QAT builds a binary partition tree over the data using feature-based or condition-aware splits, and transfers the same tree topology to the Gaussian source by quantile alignment, yielding the structured coupling $\pi_{\mathrm{QAT}}$.
    (\emph{Right}) each data leaf is paired with an aligned Gaussian source region. Since these source regions are axis-aligned hyperrectangles, sampling from the corresponding QAT source conditional reduces to parallel per-coordinate truncated normal sampling.
    }
    \label{fig:pipeline}
\end{figure}

\begin{definition}[QAT Coupling]
\label{def:qat_coupling}
Let $\mathcal{L}$ be the nonempty leaf node set of data tree $T(\widehat{\mu}_{\mathcal{D}})$. For each leaf $\ell\in\mathcal{L}$, let $\mathcal{D}_\ell=\mathcal{D}\cap R_\ell$ and $\alpha_\ell:=\tfrac{1}{N}|\mathcal{D}_\ell|$. Let
$
\widehat{\mu}_\ell
=
|\mathcal{D}_\ell|^{-1}
\sum_{x_i\in\mathcal{D}_\ell}\delta_{x_i}
$
denote the conditional empirical measure on $\mathcal{D}_\ell$, and let $\gamma_\ell$ denote the conditional source measure associated with leaf $\ell$ in the Quantile AlignTree. Define the \emph{QAT coupling} as
\begin{equation}
\label{eq:qat_coupling}
\pi_{\mathrm{QAT}}
:=
\sum_{\ell\in\mathcal{L}}
\alpha_\ell\,\gamma_\ell\otimes\widehat{\mu}_\ell.
\end{equation}
Then $\pi_{\mathrm{QAT}}\in\Pi(\gamma,\widehat{\mu}_{\mathcal{D}})$.
\end{definition}

Intuitively, the QAT coupling matches each data leaf $\mathcal{D}_\ell$ with an aligned source region. To draw a training pair $(x_0,x_1)\sim\pi_{\mathrm{QAT}}$, one samples a leaf $\ell$ with probability $\alpha_\ell$, draws $x_0\sim\gamma_\ell$ from the corresponding conditional source distribution, and draws $x_1\sim\widehat{\mu}_\ell$ from the conditional empirical distribution on that leaf.

\begin{remark}[Gaussian Sampling]
\label{remark:gaussian}
When $\gamma=\mathcal{N}(0,I_d)$, sampling from $\gamma_\ell$ is especially convenient. Since all splits are axis-aligned, every leaf source region is an axis-aligned hyperrectangle. If leaf $\ell$ corresponds to source region $\widetilde{R}_\ell=\prod_{j=1}^d[a_j^{(\ell)},b_j^{(\ell)}]$, sampling from $\gamma_\ell$ reduces to drawing each coordinate independently from a truncated standard normal:
\[
x_0^{(j)}
\sim
\mathcal{N}(0,1)\big|_{[a_j^{(\ell)},b_j^{(\ell)}]},
\qquad j=1,\ldots,d.
\]
The interval endpoints are determined by all split constraints along the root-to-leaf path, and the sampling is fully parallelizable across coordinates.
\end{remark}

\subsection{QAT Flow Matching}
\label{subsec:qatfm}

Substituting the QAT coupling into the linear FM objective~\eqref{eq:linear-cfm} yields the QAT-FM training objective. Let $\widehat{\mu}_{\mathcal{D}}$ be the empirical data distribution, $\gamma$ the source distribution, and $\pi_{\mathrm{QAT}}\in\Pi(\gamma,\widehat{\mu}_{\mathcal{D}})$ the QAT coupling of Definition~\ref{def:qat_coupling}. QAT-FM learns the velocity field via
\begin{equation}
\label{eq:qatfm}
\mathcal{L}_{\mathrm{QAT\text{-}FM}}(\theta)
:=
\EE_{\substack{t\sim\mathrm{Unif}[0,1]\\ (x_0,x_1)\sim\pi_{\mathrm{QAT}}}}
\left\|v_\theta\bigl(t,(1-t)x_0+tx_1\bigr)-(x_1-x_0)\right\|^2.
\end{equation}
Compared to standard FM, QAT-FM differs only in the sample pairing: $\pi_{\mathrm{QAT}}$ replaces the independent coupling, while the linear interpolation path and velocity regression target remain unchanged. This structured pairing yields the following geometric properties.

\begin{theorem}[Non-Crossing Property]
\label{thm:noncrossing}
Let $\pi_{\mathrm{QAT}}\in\Pi(\gamma,\widehat{\mu}_{\mathcal{D}})$ be the QAT coupling. Consider two paired samples $(x_0,x_1),(x_0',x_1')\in\mathrm{supp}(\pi_{\mathrm{QAT}})$. If $x_1$ and $x_1'$ belong to two distinct leaf nodes $\ell\neq\ell'$ of the data tree, then the corresponding linear interpolation paths
\[
x_t=(1-t)x_0+tx_1, \qquad x_t'=(1-t)x_0'+tx_1'
\]
do not intersect, i.e., $x_t\neq x_t'$ for all $t\in[0,1]$, almost surely.
\end{theorem}

Theorem~\ref{thm:noncrossing} establishes that paths associated with distinct target leaves are strictly non-crossing. This property alleviates the velocity ambiguity that can arise from overlapping training paths: since the FM objective fits $v_\theta$ pointwise at each $t\in[0,1]$, two training paths satisfying $x_t=x_t'$ but $x_1-x_0\neq x_1'-x_0'$ impose conflicting supervision signals at the same spatial location, increasing the difficulty of fitting $v_\theta$. Paths that remain well-separated at intermediate times are therefore desirable. Beyond preventing cross-leaf path crossings, the QAT coupling also consistently improves the expected inter-path distance, as the following theorem quantifies.

\begin{theorem}[Improved Path Separation]
\label{thm:path-separation}
Let $\gamma=\mathcal{N}(0,I_d)$, let $\pi_{\mathrm{QAT}}\in\Pi(\gamma,\widehat{\mu}_{\mathcal{D}})$ be the QAT coupling, and let $\pi_{\mathrm{ind}}:=\gamma\otimes\widehat{\mu}_{\mathcal{D}}$ be the independent coupling. For any coupling $\pi$, let $\EE_{\pi^{\otimes 2}}$ denote expectation over two independent pairs $(x_0,x_1),(x_0',x_1')\sim\pi$, and define $x_t=(1-t)x_0+tx_1$ and $x_t'=(1-t)x_0'+tx_1'$. Then for every $t\in(0,1)$,
\[
\EE_{\pi_{\mathrm{QAT}}^{\otimes 2}}\!\left[\|x_t-x_t'\|_2^2\right]
-
\EE_{\pi_{\mathrm{ind}}^{\otimes 2}}\!\left[\|x_t-x_t'\|_2^2\right]
=
4t(1-t)\,\Delta_{\mathrm{QAT}},
\]
where $\Delta_{\mathrm{QAT}}$ is the node-level path-separation gain induced by the QAT tree:
\begin{equation}
\label{eq:qat_lb}
\Delta_{\mathrm{QAT}}
:=
\sum_{\eta\in\mathcal{I}}
\alpha_\eta q_\eta(1-q_\eta)
\Bigl(\EE_{\gamma_{\eta^+}}\!\bigl[X^{(\kappa_\eta)}\bigr] - \EE_{\gamma_{\eta^-}}\!\bigl[X^{(\kappa_\eta)}\bigr]\Bigr)
\Bigl(\EE_{\widehat{\mu}_{\eta^+}}\!\bigl[X^{(\kappa_\eta)}\bigr] - \EE_{\widehat{\mu}_{\eta^-}}\!\bigl[X^{(\kappa_\eta)}\bigr]\Bigr).
\end{equation}
Here $\alpha_\eta:=\widehat{\mu}_{\mathcal{D}}(R_\eta)$ is the data mass at node $\eta$, $q_\eta:=F_{\widehat{\mu}_\eta}^{(\kappa_\eta)}(s_\eta)$ is the conditional left-child mass fraction at split coordinate $\kappa_\eta$, and $\gamma_{\eta^\pm}$, $\widehat{\mu}_{\eta^\pm}$ denote the conditional source and empirical distributions on the left and right child regions of $\eta$. Under non-degenerate splits, $\Delta_{\mathrm{QAT}}>0$, so the QAT coupling strictly dominates the independent coupling in expected inter-path separation.
\end{theorem}

\subsection{Conditional QAT Strategy}
\label{sec:conditional}

In conditional generation tasks, data samples carry conditioning variables. Let $\{(x_i,c_i)\}_{i=1}^N\subseteq\RR^d\times\RR^m$ be the joint dataset, $\widehat{\mu}_{\mathcal{D}_C}:=N^{-1}\sum_{i=1}^N\delta_{(x_i,c_i)}$ the joint empirical distribution, and $\widehat{\mu}_C:=N^{-1}\sum_{i=1}^N\delta_{c_i}$ the marginal over condition space. A standard OT coupling based only on data-space distances can ignore the alignment between samples and their conditioning variables. Alternatively, solving OT directly between the joint source $\gamma\otimes\widehat{\mu}_C$ and the joint empirical distribution $\widehat{\mu}_{\mathcal{D}_C}$ increases the optimization dimension and computational cost~\citep{cheng2025curse,mousavi2026flow}. We therefore propose the following Conditional QAT strategy.

\begin{definition}[Conditional QAT]
\label{def:conditional_qat}
Let $\{c_i\}_{i=1}^N\subseteq\RR^m$ be conditioning variables. Conditional QAT constructs a data tree on the joint samples $\{(x_i,c_i)\}_{i=1}^N\subseteq\RR^{d+m}$. For each internal node $\eta$, let $\kappa_\eta\in[d+m]$ denote its split coordinate. When building the source-side AlignTree, if $\kappa_\eta\le d$ (a data coordinate), the source side is partitioned using the standard QAT aligned threshold in~\eqref{eq:qat_split}; if $\kappa_\eta>d$ (a condition coordinate), no split is applied to the source and both child nodes inherit the parent source region, the split is used only to refine the target-side grouping: $\widetilde{R}_{\eta^-}=\widetilde{R}_{\eta^+}=\widetilde{R}_\eta$.
\end{definition}

The Conditional QAT coupling takes the same form as Definition~\ref{def:qat_coupling}, with the target empirical distribution replaced by the joint distribution $\widehat{\mu}_{\mathcal{D}_C}$:
\[
\pi_{\mathrm{CQAT}}
=
\sum_{\ell\in\mathcal{L}_C}
\alpha_\ell\,\gamma_\ell\otimes\widehat{\mu}_{\mathcal{D}_C,\ell}
\in
\Pi(\gamma,\widehat{\mu}_{\mathcal{D}_C}),
\]
where $\mathcal{L}_C$ is the set of nonempty leaves of the joint data tree, and $\widehat{\mu}_{\mathcal{D}_C,\ell}$ is the conditional joint empirical distribution at leaf $\ell$. 
A sample from $\pi_{\mathrm{CQAT}}$ can be written as $(x_0,(x_1,c))$, where the linear FM path is applied only in the data space between $x_0$ and $x_1$, while $c$ is provided as conditioning input to the velocity model. 
By encoding conditional structure through leaf-level matching in the joint tree, this construction avoids explicitly solving the product-space OT problem between $\gamma\otimes\widehat{\mu}_C$ and $\widehat{\mu}_{\mathcal{D}_C}$, while retaining the efficient sampling form of unconditional QAT.

\begin{remark}[Theoretical extensions]
\label{remark:cqat_extensions}
Conditional QAT inherits the theoretical properties of QAT under appropriate conditions. Specifically, if the root-to-$\ell$ and root-to-$\ell'$ paths first diverge at a data-coordinate split node (i.e., $\kappa_\eta\le d$), then the linear interpolation paths drawn from the corresponding leaf couplings satisfy the non-crossing property of Theorem~\ref{thm:noncrossing}. Under the same non-degeneracy conditions, Conditional QAT also inherits the path-separation bound of Theorem~\ref{thm:path-separation}. Proofs and further details are given in Appendix~\ref{subsec:cqat_properties}.
\end{remark}

\subsection{Efficient QAT Construction and Sampling}
\label{subsec:alg}

When $\gamma=\mathcal{N}(0,I_d)$, each source-side leaf region $\widetilde{R}_\ell$ is fully characterized by its coordinate-wise bounds $(a_\ell,b_\ell)$ (Remark~\ref{remark:gaussian}). Regardless of whether the data tree is built from unconditional or conditional samples, each leaf $\ell$ stores only the data index set $\mathcal{S}_\ell$ and the source-side bounds $(a_\ell,b_\ell)$, which fully specify the leaf-wise product coupling $\gamma_\ell\otimes\widehat{\mu}_\ell$. Training pairs are drawn by sampling a leaf $\ell$ proportionally to its mass, then independently drawing a source point from the corresponding truncated Gaussian region and a target point from $\mathcal{S}_\ell$. The full construction, sampling algorithm, and complexity analysis are given in Appendix~\ref{sec:alg}.

\paragraph{Split Strategy.}
QAT is compatible with a variety of split strategies. Throughout this work we adopt a maximum-variance rule: for each node $\eta$, the split coordinate is the dimension of highest empirical variance and the split threshold is the empirical mean along that coordinate,
\[
\kappa_\eta
=
\arg\max_j\operatorname{Var}_{\widehat{\mu}_\eta}\!\big(Y^{(j)}\big),
\qquad
s_\eta
=
\EE_{\widehat{\mu}_\eta}\!\big[Y^{(\kappa_\eta)}\big],
\]
where $Y=X$ in the unconditional case and $Y=(X,C)$ in the conditional case. For discrete empirical data, we recurse until each leaf contains at most one sample; in the conditional setting, the terminal splitting phase prioritizes data coordinates over condition coordinates, maximizing the number of source-side splits and thereby increasing the path-separation gain of Theorem~\ref{thm:path-separation}.

\paragraph{High-Dimensional Data Handling.}
For high-dimensional data, axis-aligned partitioning in the raw data space may fail to capture the principal directions of variation.
We therefore first rotate the data via an orthogonal transform $P\in\RR^{d\times d}$, setting $z_i=Px_i$ and constructing the QAT coupling on the top-$k$ subspace $z_i^{(1:k)}$. By the rotational invariance of the standard Gaussian, $P_\#\mathcal{N}(0,I_d)=\mathcal{N}(0,I_d)$ for any orthogonal $P$. Training then proceeds by sampling $(\epsilon^{(1:k)},z_i^{(1:k)})\sim\pi_{\mathrm{QAT}}$, independently drawing $\epsilon^{(-k)}\sim\mathcal{N}(0,I_{d-k})$, and setting
\[
x_0=P^\top\!\big(\epsilon^{(1:k)},\epsilon^{(-k)}\big).
\]
The marginal distribution of $x_0$ remains $\mathcal{N}(0,I_d)$, while the QAT tree operates entirely within the low-dimensional subspace. In practice, $P$ is composed of a Patch-Hadamard transform and a PCA rotation at the mean-pooling scale, concentrating the principal data variation into a small number of coordinates. Details are given in Appendix~\ref{subsec:hd_alg}.

\section{Experiments}
\label{sec:exp}

We evaluate QAT-FM through a series of experiments: we begin with path visualizations on two-dimensional synthetic data, then proceed to high-dimensional image generation benchmarks covering unconditional generation (CIFAR-10), class-conditional generation (ImageNet-32), text-conditional generation (CelebA-64), and large-scale latent-space generation (ImageNet-256). Additional implementation details are provided in Appendix~\ref{sec:imple}; extended experimental results appear in Appendix~\ref{sec:exp_add}.

\subsection{2D Synthetic Experiments}
\label{subsec:2d}

\begin{figure}[t]
    \centering
    \includegraphics[width=\linewidth]{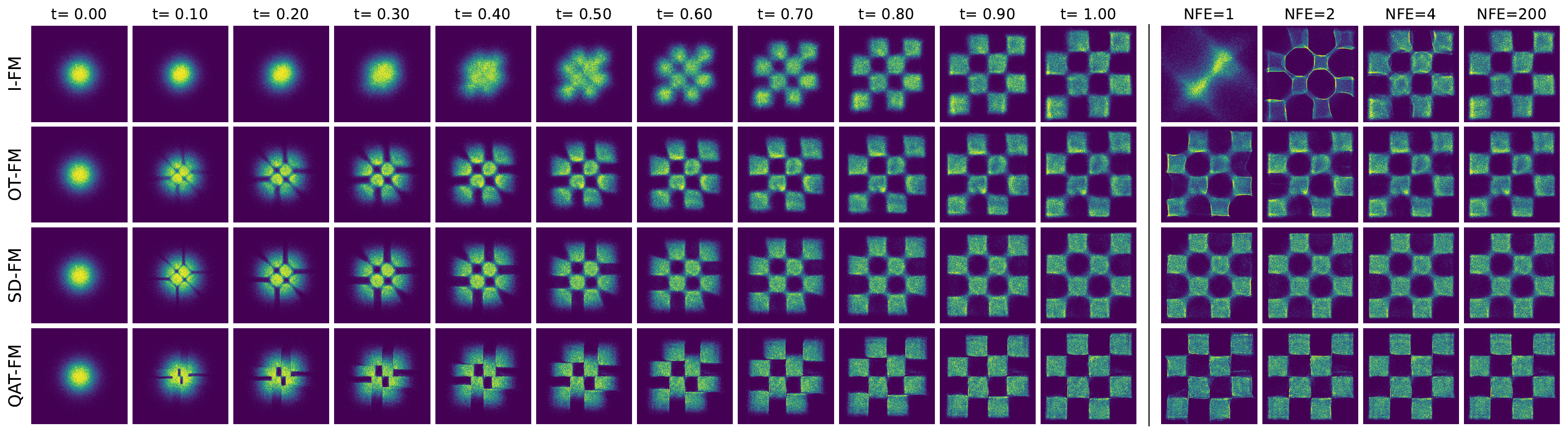}
    \caption{\footnotesize 
    \textbf{2D checkerboard results.}
    (\emph{Left}) sample evolution induced by different coupling methods from $t=0$ to $t=1$.
    (\emph{Right}) generated samples with varying numbers of function evaluations (NFE).
    }
    \label{fig:2d_results}
\end{figure}

Following~\citet{lipman2022flow,pooladian2023multisample}, we compare different coupling methods on two-dimensional checkerboard data. Four methods are evaluated: I-FM, OT-FM~\citep{tong2024improving}, SD-FM~\citep{mousavi2026flow}, and the proposed QAT-FM. All methods share the same network architecture and training configuration.

Figure~\ref{fig:2d_results} (left) shows sample evolution from $t=0$ to $t=1$. OT-FM and SD-FM capture the global structure of the target distribution earlier than I-FM, while QAT-FM yields better local separation and more organized intermediate transport. Figure~\ref{fig:2d_results} (right) shows generated samples with varying NFE. QAT-FM recovers a clear checkerboard structure at small NFE, indicating that its structured coupling improves few-step sample quality.

\subsection{Unconditional Generation on CIFAR-10}
\label{sec:cifar10}

We evaluate QAT-FM on unconditional generation using CIFAR-10~\citep{krizhevsky2009learning}, which contains $50{,}000$ RGB training images of size $32\times32$ across $10$ categories. The velocity-field network follows the architecture of~\citet{tong2024improving}; training details are provided in Appendix~\ref{sec:imple}. We generate $50{,}000$ samples using several numerical solvers and compute FID~\citep{heusel2017gans}. Results are averaged over $5$ random seeds and reported with standard deviations.

\begin{figure}[ht]
\centering

\begin{minipage}[t]{0.62\linewidth}
\centering
\captionof{table}{FID results on unconditional CIFAR-10 generation at $100$k training steps.}
\label{tab:fid_cifar10_short}
\small
\setlength{\tabcolsep}{4pt}
\renewcommand{\arraystretch}{1.05}
\resizebox{\linewidth}{!}{
\begin{tabular}{lcccc}
\toprule
Method & Euler-2 & Euler-100 & Dopri5 & NFE$\downarrow$ \\
\midrule
I-FM
& 164.18$_{\pm \text{0.13}}$
& \textbf{3.79}$_{\pm \textbf{0.04}}$
& \underline{3.38$_{\pm \text{0.03}}$}
& 137.30$_{\pm \text{0.30}}$ \\
OT-FM
& \textbf{77.26}$_{\pm \textbf{0.12}}$
& 3.96$_{\pm \text{0.04}}$
& 3.49$_{\pm \text{0.03}}$
& \textbf{129.52}$_{\pm \textbf{0.86}}$ \\
SD-FM
& 162.66$_{\pm \text{0.23}}$
& 4.03$_{\pm \text{0.02}}$
& 3.61$_{\pm \text{0.01}}$
& \underline{135.78$_{\pm \text{0.36}}$} \\
\rowcolor{gray!15}
QAT-FM
& \underline{98.68$_{\pm \text{0.17}}$}
&  \underline{3.91}$_{\pm \text{0.04}}$
& \textbf{3.26}$_{\pm \textbf{0.02}}$
& 145.55$_{\pm \text{1.05}}$ \\
\bottomrule
\end{tabular}
}
\end{minipage}
\hfill
\begin{minipage}[t]{0.37\linewidth}
\centering
\vspace{0pt}
\IfFileExists{figures/CIFAR10_step_fid.pdf}{%
  \includegraphics[width=\linewidth]{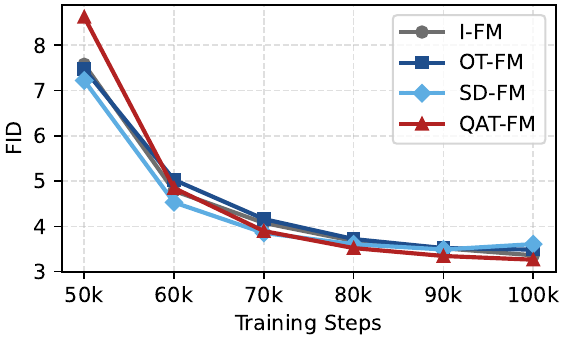}%
}{%
  \fbox{\parbox[c][0.6\linewidth][c]{\linewidth}{\centering\small[Figure: FID vs.\ steps]}}%
}
\caption{\footnotesize FID (Dopri5) vs. steps.}
\label{fig:cifar10_step_fid}
\end{minipage}

\end{figure}
Table~\ref{tab:fid_cifar10_short} reports CIFAR-10 results at $100$k training steps. Under few-step Euler sampling, OT-FM performs best, while QAT-FM substantially improves over I-FM and SD-FM, indicating that structured coupling can improve generation quality under coarse numerical integration. Under the adaptive {Dopri5} solver, QAT-FM achieves the lowest FID, showing competitive final generation quality in the unconditional setting. Figure~\ref{fig:cifar10_step_fid} further shows that QAT-FM maintains strong performance during late-stage training.

\paragraph{Ablation Studies.}
Our main experiments use patch size $p=4$, PCA dimension $k=32$, and mean-based tree splitting. To assess the sensitivity of these choices, we conduct three ablation studies: (1) varying the PCA dimension $k\in\{32,64,128\}$; (2) varying the patch size $p\in\{1,2,4\}$; and (3) comparing three split strategies---mean split, median split, and random split with ratio $\sim\mathcal{U}(0.25,0.75)$. Results are summarized in Table~\ref{tab:ablation}.
\begin{table}[ht]
\centering
\caption{\footnotesize Ablation studies on PCA dimension, patch size, and tree splitting strategy.}
\label{tab:ablation}
\small
\setlength{\tabcolsep}{4pt}
\renewcommand{\arraystretch}{1.05}

\begin{minipage}[t]{0.31\linewidth}
\centering
\begin{tabular}{lcc}
\toprule
PCA Dim. & FID$\downarrow$ & NFE$\downarrow$ \\
\midrule
$k=32$  & 3.26 & 145.48 \\
$k=64$  & 3.27 & 143.50 \\
$k=128$ & 3.24 & 143.76 \\
\bottomrule
\end{tabular}
\end{minipage}
\hfill
\begin{minipage}[t]{0.31\linewidth}
\centering
\begin{tabular}{lcc}
\toprule
Patch Size & FID$\downarrow$ & NFE$\downarrow$ \\
\midrule
$p=1$ & 3.24 & 145.36 \\
$p=2$ & 3.24 & 145.55 \\
$p=4$ & 3.26 & 145.48 \\
\bottomrule
\end{tabular}
\end{minipage}
\hfill
\begin{minipage}[t]{0.31\linewidth}
\centering
\begin{tabular}{lcc}
\toprule
Split Rule & FID$\downarrow$ & NFE$\downarrow$ \\
\midrule
Mean   & 3.26 & 145.48 \\
Median & 3.22 & 143.89 \\
Random & 3.24 & 144.87 \\
\bottomrule
\end{tabular}
\end{minipage}

\end{table}

As shown in Table~\ref{tab:ablation}, QAT-FM is insensitive to the choice of PCA dimension, patch size, and tree splitting strategy; FID and NFE vary only slightly across all configurations. These results suggest that the performance gains of QAT-FM do not rely on carefully tuned tree construction hyperparameters, but stem primarily from the global tree-structured quantile alignment mechanism.

\subsection{Conditional Generation Experiments}

\label{subsec:cond_gen}

\paragraph{Class-conditional generation.}
We evaluate QAT-FM on class-conditional generation using ImageNet-1k~\citep{deng2009imagenet}, which contains $1{,}000$ categories and over $1$M training images. All images are resized to $32\times32$, with random horizontal flipping used for augmentation. We build the Conditional QAT coupling using both original and horizontally flipped training images. At evaluation, we generate $50$K samples and compute FID against the training set. Results are averaged over $5$ random seeds and reported with standard deviations.

\begin{table}[ht]
\centering
\caption{\footnotesize FID results on class-conditional ImageNet-$32\times32$ generation.}
\label{tab:fid_imagenet32}
\small
\setlength{\tabcolsep}{4pt}
\resizebox{\linewidth}{!}{
\begin{tabular}{lcccccccc}
\toprule
Method & {Euler-2} & {Euler-5} & {Euler-10} & {Euler-20} & {Euler-50} & {Euler-100} & {Dopri5} & NFE$\downarrow$ \\
\midrule
I-FM 
& 122.40$_{\pm \text{0.08}}$
& {23.04$_{\pm \text{0.07}}$}
& \underline{10.12$_{\pm \text{0.07}}$}
& \underline{6.48$_{\pm \text{0.07}}$}
& \underline{4.77$_{\pm \text{0.07}}$}
& \underline{4.25$_{\pm \text{0.06}}$}
& \underline{3.82$_{\pm \text{0.04}}$}
& \textbf{140.27}$_{\pm \textbf{1.99}}$ \\

OT-FM 
& {115.63$_{\pm \text{0.42}}$}
& 31.82$_{\pm \text{0.11}}$
& 15.79$_{\pm \text{0.02}}$
& 10.59$_{\pm \text{0.00}}$
& 8.14$_{\pm \text{0.02}}$
& 7.37$_{\pm \text{0.00}}$
& 6.65$_{\pm \text{0.02}}$
& \underline{142.51$_{\pm \text{0.45}}$} \\

C2OT 
& \textbf{100.41$_{\pm \textbf{0.05}}$}
& \textbf{20.94}$_{\pm \textbf{0.08}}$
& \textbf{10.10}$_{\pm \textbf{0.06}}$
& 6.80$_{\pm \text{0.04}}$
& 5.15$_{\pm \text{0.04}}$
& 4.62$_{\pm \text{0.04}}$
& 4.16$_{\pm \text{0.02}}$
& {143.97$_{\pm \text{0.75}}$} \\

\rowcolor{gray!15}
QAT-FM 
& \underline{105.16$_{\pm \text{0.30}}$}
& \underline{22.88$_{\pm \textbf{0.05}}$}
& {10.30$_{\pm \text{0.08}}$}
& \textbf{6.40}$_{\pm \textbf{0.07}}$
& \textbf{4.54}$_{\pm \textbf{0.07}}$
& \textbf{3.99}$_{\pm \textbf{0.07}}$
& \textbf{3.56}$_{\pm \textbf{0.06}}$
& 143.35$_{\pm \text{2.22}}$ \\
\bottomrule
\end{tabular}
}
\end{table}

Table~\ref{tab:fid_imagenet32} reports FID on class-conditional ImageNet-$32\times32$ generation. Both C2OT and QAT-FM improve few-step generation over I-FM, demonstrating the effectiveness of condition-aware coupling. While C2OT performs best under small Euler step counts, its advantage reverses as the solver becomes more accurate. By contrast, QAT-FM remains competitive in few-step sampling and achieves the best performance under larger-step Euler solvers and Dopri5. This indicates that Conditional QAT provides more stable supervision for high-quality conditional FM training.

\paragraph{Text-conditional generation.}
We further evaluate QAT-FM on text-conditional CelebA-$64\times64$ generation. CelebA~\citep{liu2015deep} contains over $200$K celebrity face images with attribute annotations, and CelebA-Dialog~\citep{jiang2021talk} provides fine-grained natural-language descriptions for these images. We use the text captions as conditioning inputs. Following C2OT~\citep{cheng2025curse}, each caption is encoded by a CLIP~\citep{radford2021learning}-like DFN text encoder~\citep{fang2024data}, and the resulting embedding is used as the condition variable. We compare QAT-FM with I-FM, OT-FM, and C2OT under the same text-conditioning setting. At evaluation, we generate $50$K samples and report FID for image fidelity and SigLIP-2~\citep{tschannen2025siglip2} CLIP score for text-image alignment.

\begin{figure}[ht]
\centering

\begin{minipage}[t]{0.45\linewidth}
\centering
\vspace{0pt}
\small
\setlength{\tabcolsep}{3pt}
\renewcommand{\arraystretch}{1.05}
\resizebox{\linewidth}{!}{
\begin{tabular}{clcccc}
\toprule
 & Method & Euler-2 & Euler-100 & Dopri5 & NFE$\downarrow$ \\
\midrule
\multirow{4}{*}{\rotatebox[origin=c]{90}{FID$\downarrow$}}
& I-FM
& 82.40
& 3.60
& 3.13
& 149.97 \\
& OT-FM
& \underline{70.24}
& \underline{3.42}
& \underline{2.98}
& \underline{146.53} \\
& C2OT
& {81.39}
& {3.50}
& {3.08}
& {148.75} \\
& \cellcolor{gray!15}QAT-FM
& \cellcolor{gray!15}\textbf{61.84}
& \cellcolor{gray!15}\textbf{3.17}
& \cellcolor{gray!15}\textbf{2.68}
& \cellcolor{gray!15}\textbf{144.08} \\
\midrule
\multirow{4}{*}{\rotatebox[origin=c]{90}{CLIP$\uparrow$}}
& I-FM
& \textbf{0.11}
& \textbf{0.10}
& \textbf{0.10}
& 149.97 \\
& OT-FM
& 0.09
& 0.08
& 0.08 
& \underline{146.53} \\
& C2OT
& \textbf{0.11}
& \textbf{0.10}
& \textbf{0.10}
& {148.75} \\
& \cellcolor{gray!15}QAT-FM
& \cellcolor{gray!15}\textbf{0.11}
& \cellcolor{gray!15}\textbf{0.10}
& \cellcolor{gray!15}\textbf{0.10}
& \cellcolor{gray!15}\textbf{144.08} \\
\bottomrule
\end{tabular}}
\vspace{1pt}

{\footnotesize \textbf{(a)} Quantitative results.}
\end{minipage}
\hfill
\begin{minipage}[t]{0.53\linewidth}
\centering
\vspace{0pt}
\includegraphics[width=\linewidth]{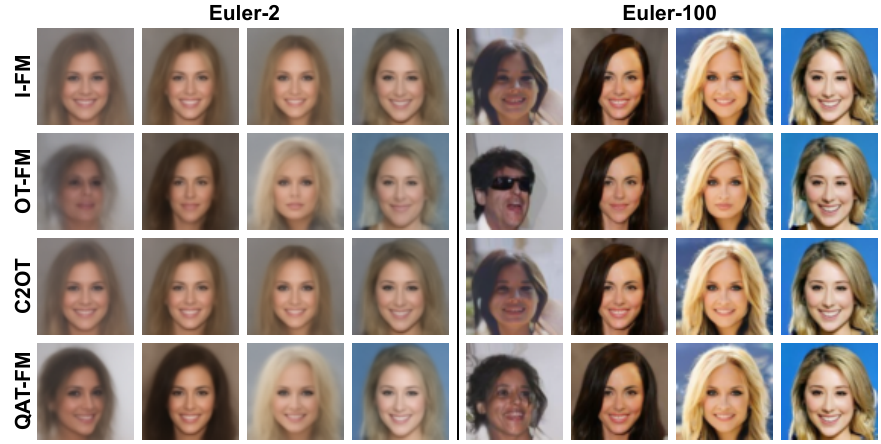}

{\footnotesize \textbf{(b)} Qualitative samples.}
\end{minipage}
\caption{\footnotesize 
Text-conditional generation results on CelebA-$64\times64$.
(a) Quantitative comparison in terms of FID and CLIP score under different solvers.
(b) Qualitative samples generated by Euler-2 and Euler-100 using the caption:
\textit{``The full face of this female is beamed with happiness. This person looks very young and has no bangs. There is not any glasses on her face.''}
}
\label{fig:celeba_quant_qual}

\end{figure}

Figure~\ref{fig:celeba_quant_qual} shows that QAT-FM consistently outperforms I-FM, OT-FM, and C2OT in FID across different solvers, while also requiring the lowest average NFE. Meanwhile, QAT-FM attains the same CLIP scores as I-FM and C2OT, and higher CLIP scores than OT-FM, indicating that the improvement in image fidelity does not come at the cost of weaker text-image alignment. The qualitative samples further show that QAT-FM produces more coherent faces under Euler-2 and sharper details under Euler-100, consistent with the quantitative results.

\begin{figure}[t]
    \centering
    \includegraphics[width=.9\linewidth]{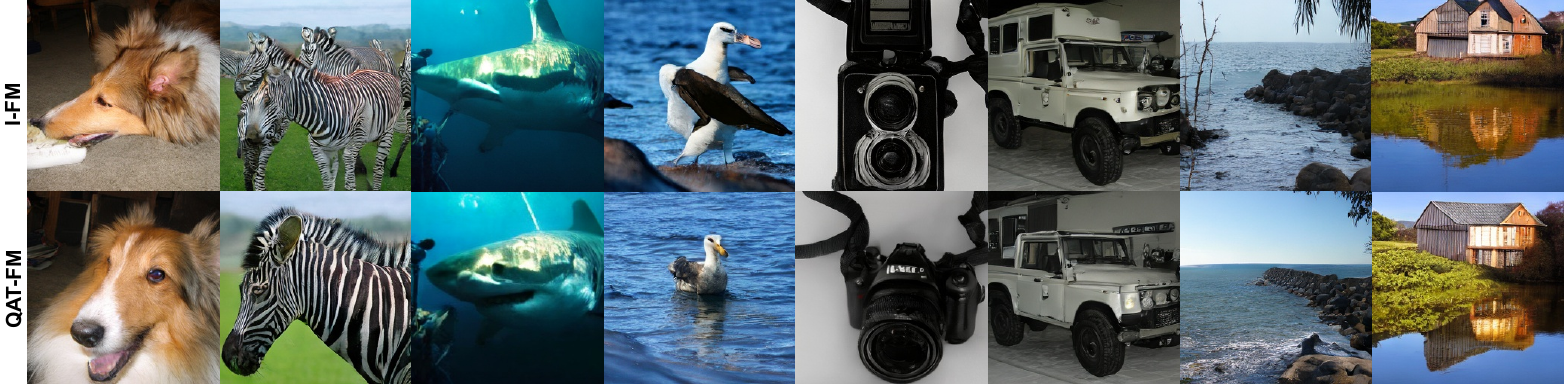}
    \caption{\footnotesize
    Qualitative comparison between I-FM and QAT-FM on ImageNet-$256\times256$ class-conditional generation using Euler-250 sampling. Each column uses the same input condition.
    }
    \label{fig:lightingdit}
\end{figure}

\begin{wraptable}[13]{r}{0.5\textwidth}
\centering
\caption{\footnotesize Class-conditional generation results on ImageNet-$256\times256$ latent generation.}
\label{tab:euler_results}
\small
\setlength{\tabcolsep}{3pt}
\resizebox{\linewidth}{!}{
\begin{tabular}{lcccccc}
\toprule
Method & Steps & IS$\uparrow$ & FID$\downarrow$ & sFID$\downarrow$ & Pre.$\uparrow$ & Rec.$\uparrow$ \\
\midrule
I-FM 
& 10
& 152.19 
& 11.83 
& \textbf{20.73} 
& 0.64
& \textbf{0.49} \\

\rowcolor{gray!15}
QAT-FM 
& 10
& \textbf{162.15} 
& \textbf{10.67} 
& 20.74 
& \textbf{0.66} 
& 0.48 \\
\midrule

I-FM 
& 50
& 242.65 
& 2.57 
& \textbf{5.62} 
& \textbf{0.80} 
& 0.57 \\
\rowcolor{gray!15}
QAT-FM 
& 50
& \textbf{250.36} 
& \textbf{2.46} 
& 5.69 
& \textbf{0.80}
& \textbf{0.58} \\
\midrule

I-FM 
& 250
& 248.94 
& 2.08 
& 4.29 
& \textbf{0.81} 
& 0.59 \\
\rowcolor{gray!15}
QAT-FM 
& 250
& \textbf{256.96} 
& \textbf{1.99} 
& \textbf{4.25} 
& 0.80
& \textbf{0.60} \\
\bottomrule
\end{tabular}
}
\end{wraptable}
\paragraph{Large-scale latent generation.}
To assess scalability, we further evaluate QAT-FM on class-conditional ImageNet-$256\times256$ latent generation~\citep{rombach2022high,ma2024sit}. Following~\citet{yao2025reconstruction}, images are encoded into VA-VAE latents, and the flow model is trained in the latent space. We construct the Conditional QAT directly over the VAE latents, while following the network architecture, training hyperparameters, and evaluation protocol of~\citet{yao2025reconstruction}. Due to computational constraints, we train for $64$ epochs using the official recommended setting. The implementation and evaluation details are provided in Appendix~\ref{sec:imple}.

Table~\ref{tab:euler_results} shows that QAT-FM consistently improves FID and Inception Score over I-FM across all solver step counts, while maintaining competitive sFID, precision, and recall. The improvements hold in both few-step and high-accuracy sampling regimes, indicating that the QAT coupling benefits latent-space flow training beyond a particular solver choice. Figure~\ref{fig:lightingdit} provides qualitative examples under Euler-250 sampling. These results suggest that QAT-FM can be integrated into modern latent-space flow backbones with minimal architectural changes while improving generation quality.

\section{Conclusion}
We study the prior-data coupling problem in Flow Matching and propose QAT-FM, a simple yet effective coupling strategy based on tree-structured quantile alignment. QAT-FM avoids the high cost of OT-based coupling while preserving favorable path geometry, including non-crossing paths and improved path separation. It also extends naturally to conditional generation through condition-aware tree construction. Extensive experiments verify the effectiveness and scalability of our method across synthetic data, image generation, and large-scale latent-space generation. Overall, QAT-FM provides an efficient and structured coupling mechanism for improving Flow Matching training.

\newpage
\bibliography{Reference}
\bibliographystyle{qatfm_preprint}

\newpage
\appendix

\section{Proofs}
\label{sec:proof}

\subsection{Proof of Theorem~\ref{thm:noncrossing}}
\label{subsec:proof_noncrossing}

\begin{proof}
By the binary tree structure, for any two distinct data leaves $\ell\neq\ell'$, there exists a unique first internal node $\eta$ at which the root-to-$\ell$ and root-to-$\ell'$ paths diverge. Without loss of generality, assume $\ell$ lies in the left subtree of $\eta$ and $\ell'$ lies in the right subtree.

By the construction of the data tree, node $\eta$ is split along coordinate $\kappa_\eta$ with threshold $s_\eta$, so its left and right child regions are
\[
R_{\eta^-}
=
\{x\in R_\eta:x^{(\kappa_\eta)}\le s_\eta\},
\qquad
R_{\eta^+}
=
\{x\in R_\eta:x^{(\kappa_\eta)}>s_\eta\}.
\]
Since $x_1\in R_{\eta^-}$ and $x_1'\in R_{\eta^+}$,
\begin{equation}
\label{eq:data-side-separation}
x_1^{(\kappa_\eta)}
\le
s_\eta
<
(x_1')^{(\kappa_\eta)}.
\end{equation}

On the source side, QAT reuses the same split coordinate $\kappa_\eta$ and replaces the data-side threshold $s_\eta$ by the aligned threshold $\widetilde{s}_\eta$. The corresponding source-side child regions are
\[
\widetilde{R}_{\eta^-}
=
\{z\in\widetilde{R}_\eta:z^{(\kappa_\eta)}\le\widetilde{s}_\eta\},
\qquad
\widetilde{R}_{\eta^+}
=
\{z\in\widetilde{R}_\eta:z^{(\kappa_\eta)}>\widetilde{s}_\eta\}.
\]
Since QAT matches each data leaf with the corresponding source leaf region, the source samples paired with leaves $\ell$ and $\ell'$ satisfy $x_0\in\widetilde{R}_{\eta^-}$ and $x_0'\in\widetilde{R}_{\eta^+}$. Hence
\begin{equation}
\label{eq:source-side-separation}
x_0^{(\kappa_\eta)}
\le
\widetilde{s}_\eta
<
(x_0')^{(\kappa_\eta)}.
\end{equation}

For any $t\in[0,1]$, the two linear interpolation paths are
\[
x_t=(1-t)x_0+tx_1,
\qquad
x_t'=(1-t)x_0'+tx_1'.
\]
Examining the $\kappa_\eta$-th coordinate,
\[
(x_t')^{(\kappa_\eta)}-x_t^{(\kappa_\eta)}
=
(1-t)\bigl\{(x_0')^{(\kappa_\eta)}-x_0^{(\kappa_\eta)}\bigr\}
+
t\bigl\{(x_1')^{(\kappa_\eta)}-x_1^{(\kappa_\eta)}\bigr\}.
\]
By~\eqref{eq:data-side-separation} and~\eqref{eq:source-side-separation}, both differences on the right-hand side are strictly positive. Therefore,
\[
(x_t')^{(\kappa_\eta)}-x_t^{(\kappa_\eta)}>0,
\qquad \forall\,t\in[0,1],
\]
so $x_t'\neq x_t$ for all $t\in[0,1]$. This proves that linear interpolation paths from two distinct QAT leaves are strictly non-coincident almost surely.
\end{proof}

\subsection{Proof of Theorem~\ref{thm:path-separation}}
\label{subsec:proof_path_separation}

\begin{proof}
For any coupling $\pi\in\Pi(\gamma,\widehat{\mu}_{\mathcal{D}})$, let $(X_0,X_1)$ and $(X_0',X_1')$ be two independent pairs drawn from $\pi$. By the linear interpolation $X_t-X_t'=(1-t)(X_0-X_0')+t(X_1-X_1')$,
\begin{equation}
\label{eq:proof_path_expand}
\begin{aligned}
\EE_{\pi^{\otimes 2}}\|X_t-X_t'\|_2^2
=&\,
(1-t)^2\EE\|X_0-X_0'\|_2^2
+
t^2\EE\|X_1-X_1'\|_2^2  \\
&+
2t(1-t)
\EE_{\pi^{\otimes 2}}\!\left[
(X_0-X_0')^\top(X_1-X_1')
\right].
\end{aligned}
\end{equation}
The first two terms depend only on the marginal distributions $\gamma$ and $\widehat{\mu}_{\mathcal{D}}$, and are therefore identical under $\pi_{\mathrm{QAT}}$ and $\pi_{\mathrm{ind}}$. The difference between the two couplings arises entirely from the cross term. Under the independent coupling $\pi_{\mathrm{ind}}=\gamma\otimes\widehat{\mu}_{\mathcal{D}}$, $X_0-X_0'$ and $X_1-X_1'$ are mutually independent with zero mean, giving
\begin{equation}
\label{eq:proof_ind_cross_zero}
\EE_{\pi_{\mathrm{ind}}^{\otimes 2}}\!\left[
(X_0-X_0')^\top(X_1-X_1')
\right]=0.
\end{equation}

\textbf{Recursive analysis of the QAT cross term.}
For each node $\eta$, let $\pi_\eta$ denote the normalized QAT coupling conditioned on node $\eta$'s region, and let $\EE_\eta$ denote expectation when both pairs are drawn independently from $\pi_\eta$. Define
\[
\Gamma_\eta
:=
\EE_\eta\!\left[
(X_0-X_0')^\top(X_1-X_1')
\right].
\]
If $\eta$ is a leaf, the QAT coupling within that leaf is $\gamma_\eta\otimes\widehat{\mu}_\eta$, so $X_0-X_0'$ and $X_1-X_1'$ are independent with zero mean, giving $\Gamma_\eta=0$.

For any internal node $\eta\in\mathcal{I}$ with left and right children $\eta^-$ and $\eta^+$, let
\[
Z_\eta:=\mathbf{1}\{X_0\in\widetilde{R}_{\eta^-},\,X_1\in R_{\eta^-}\},
\qquad
Z_\eta':=\mathbf{1}\{X_0'\in\widetilde{R}_{\eta^-},\,X_1'\in R_{\eta^-}\}
\]
indicate whether the first and second pairs fall into the left child. Write $p_\eta^-:=\widehat{\mu}_\eta(R_{\eta^-})$ and $p_\eta^+:=\widehat{\mu}_\eta(R_{\eta^+})$. By the law of total expectation,
\begin{equation}
\label{eq:proof_node_decomp_simple}
\begin{aligned}
\Gamma_\eta
=&\,
\EE\!\left[
\EE_\eta\!\left[
(X_0-X_0')^\top(X_1-X_1')
\,\middle|\,
Z_\eta,Z_\eta'
\right]
\right]  \\
=&\,
(p_\eta^-)^2\Gamma_{\eta^-}
+
(p_\eta^+)^2\Gamma_{\eta^+}  \\
&+
p_\eta^-p_\eta^+\,
\EE_\eta\!\left[
(X_0-X_0')^\top(X_1-X_1')
\,\middle|\,
Z_\eta=1,\,Z_\eta'=0
\right]  \\
&+
p_\eta^+p_\eta^-\,
\EE_\eta\!\left[
(X_0-X_0')^\top(X_1-X_1')
\,\middle|\,
Z_\eta=0,\,Z_\eta'=1
\right].
\end{aligned}
\end{equation}

When both pairs fall into the same child, they contribute $(p_\eta^-)^2\Gamma_{\eta^-}$ and $(p_\eta^+)^2\Gamma_{\eta^+}$ respectively. For the cross-subtree terms, write $m_{0,\eta}^{\pm}:=\EE_{\gamma_{\eta^\pm}}[X]$ and $m_{1,\eta}^{\pm}:=\EE_{\widehat{\mu}_{\eta^\pm}}[X]$. Taking the case $Z_\eta=0$, $Z_\eta'=1$ (first pair in right child, second in left), the cross-subtree contribution is
\[
\begin{aligned}
&p_\eta^+p_\eta^-\,
\EE_\eta\!\left[
(X_0-X_0')^\top(X_1-X_1')
\,\middle|\,
Z_\eta=0,\,Z_\eta'=1
\right]  \\
=&\,
p_\eta^+p_\eta^-
\bigg\{
\EE_{\eta^+}[X_0^\top X_1]
-
(m_{0,\eta}^+)^\top m_{1,\eta}^-
-
(m_{0,\eta}^-)^\top m_{1,\eta}^+
+
\EE_{\eta^-}[X_0^\top X_1]
\bigg\}.
\end{aligned}
\]
Since $\Gamma_{\eta^\pm}=2\bigl\{\EE_{\eta^\pm}[X_0^\top X_1]-(m_{0,\eta}^\pm)^\top m_{1,\eta}^\pm\bigr\}$, the above simplifies to
\[
p_\eta^+p_\eta^-
\bigg\{
\frac{1}{2}\Gamma_{\eta^+}
+
\frac{1}{2}\Gamma_{\eta^-}
+
\left(m_{0,\eta}^+-m_{0,\eta}^-\right)^\top
\left(m_{1,\eta}^+-m_{1,\eta}^-\right)
\bigg\}.
\]
The case $Z_\eta=1$, $Z_\eta'=0$ contributes the same amount. Substituting into~\eqref{eq:proof_node_decomp_simple},
\begin{equation}
\label{eq:proof_node_recursion_inner}
\Gamma_\eta
=\,
p_\eta^-\Gamma_{\eta^-}
+
p_\eta^+\Gamma_{\eta^+}
+
2p_\eta^-p_\eta^+
\left(m_{0,\eta}^+-m_{0,\eta}^-\right)^\top
\left(m_{1,\eta}^+-m_{1,\eta}^-\right).
\end{equation}

\textbf{Reduction to split-coordinate mean differences.}
Since the source distribution is $\mathcal{N}(0,I_d)$ and all QAT splits are axis-aligned, the source regions $\widetilde{R}_{\eta^-}$ and $\widetilde{R}_{\eta^+}$ differ only in the $\kappa_\eta$-th coordinate interval. By Gaussian coordinate independence, for any $j\neq\kappa_\eta$,
\[
(\mathrm{proj}_j)_\#\gamma_{\eta^+}
=
(\mathrm{proj}_j)_\#\gamma_{\eta^-},
\]
so $(m_{0,\eta}^+)^{(j)}=(m_{0,\eta}^-)^{(j)}$ for all $j\neq\kappa_\eta$. Writing
\[
\begin{aligned}
\Delta_\eta
&:=
\left(m_{0,\eta}^+-m_{0,\eta}^-\right)^\top
\left(m_{1,\eta}^+-m_{1,\eta}^-\right) \\
&=\,
\left[
(m_{0,\eta}^+)^{(\kappa_\eta)}
-
(m_{0,\eta}^-)^{(\kappa_\eta)}
\right]
\left[
(m_{1,\eta}^+)^{(\kappa_\eta)}
-
(m_{1,\eta}^-)^{(\kappa_\eta)}
\right],
\end{aligned}
\]
equation~\eqref{eq:proof_node_recursion_inner} becomes
\begin{equation}
\label{eq:proof_node_recursion}
\Gamma_\eta
=
p_\eta^-\Gamma_{\eta^-}
+
p_\eta^+\Gamma_{\eta^+}
+
2p_\eta^-p_\eta^+\Delta_\eta.
\end{equation}

\textbf{Telescoping from the root.}
Let $\alpha_\eta:=\widehat{\mu}_{\mathcal{D}}(R_\eta)$, $\alpha_{\eta^-}:=\widehat{\mu}_{\mathcal{D}}(R_{\eta^-})$, $\alpha_{\eta^+}:=\widehat{\mu}_{\mathcal{D}}(R_{\eta^+})$. Since $p_\eta^-=\alpha_{\eta^-}/\alpha_\eta$ and $p_\eta^+=\alpha_{\eta^+}/\alpha_\eta$, multiplying~\eqref{eq:proof_node_recursion} by $\alpha_\eta$ gives
\begin{equation}
\label{eq:proof_weighted_recursion}
\alpha_\eta\Gamma_\eta
=
\alpha_{\eta^-}\Gamma_{\eta^-}
+
\alpha_{\eta^+}\Gamma_{\eta^+}
+
2
\frac{\alpha_{\eta^-}\alpha_{\eta^+}}{\alpha_\eta}
\Delta_\eta.
\end{equation}
Telescoping from the root $r$ (where $\alpha_r=1$) and using $\Gamma_\ell=0$ at all leaves,
\begin{equation}
\label{eq:proof_cross_final}
\EE_{\pi_{\mathrm{QAT}}^{\otimes 2}}
\!\left[
(X_0-X_0')^\top(X_1-X_1')
\right]
=
\Gamma_r
=
2\sum_{\eta\in\mathcal{I}}
\frac{\alpha_{\eta^-}\alpha_{\eta^+}}{\alpha_\eta}
\Delta_\eta.
\end{equation}
Setting $q_\eta:=\alpha_{\eta^-}/\alpha_\eta$, so that $\frac{\alpha_{\eta^-}\alpha_{\eta^+}}{\alpha_\eta}=\alpha_\eta q_\eta(1-q_\eta)$,
\begin{equation}
\label{eq:proof_cross_exact}
\EE_{\pi_{\mathrm{QAT}}^{\otimes 2}}
\!\left[
(X_0-X_0')^\top(X_1-X_1')
\right]
=
2\Delta_{\mathrm{QAT}}.
\end{equation}

Combining~\eqref{eq:proof_path_expand},~\eqref{eq:proof_ind_cross_zero}, and~\eqref{eq:proof_cross_exact},
\begin{equation}
\label{eq:proof_final}
\begin{aligned}
&\EE_{\pi_{\mathrm{QAT}}^{\otimes 2}}\!\left[\|X_t-X_t'\|_2^2\right]
-
\EE_{\pi_{\mathrm{ind}}^{\otimes 2}}\!\left[\|X_t-X_t'\|_2^2\right]  \\
=&\,
2t(1-t)
\EE_{\pi_{\mathrm{QAT}}^{\otimes 2}}
\!\left[
(X_0-X_0')^\top(X_1-X_1')
\right]  \\
=&\,
4t(1-t)\Delta_{\mathrm{QAT}}.
\end{aligned}
\end{equation}
Under non-degenerate splits, each effective split node satisfies $\Delta_\eta\ge 0$; if at least one node has $\Delta_\eta>0$, then $\Delta_{\mathrm{QAT}}>0$. This completes the proof.
\end{proof}

\subsection{Properties of CQAT}
\label{subsec:cqat_properties}

Let $\mathcal{D}_c=\{(x_i,c_i)\}_{i=1}^N\subseteq\RR^d\times\RR^m$ be the conditional dataset with joint empirical distribution $\widehat{\mu}_{\mathcal{D}_c}$. Denote the Conditional QAT coupling by $\pi_{\mathrm{CQAT}}\in\Pi(\gamma,\widehat{\mu}_{\mathcal{D}_c})$, where $\gamma=\mathcal{N}(0,I_d)$. Throughout, $X_1$ refers to the data-coordinate component of a joint sample $(X_1,C)$.

\begin{proposition}[Non-crossing property of CQAT]
\label{prop:cqat_noncrossing}
Let $\ell\neq\ell'$ be two leaf nodes of the Conditional QAT tree, and let $\eta$ be the first internal node at which the root-to-$\ell$ and root-to-$\ell'$ paths diverge. If $\eta$ is a data-coordinate split node, i.e., $\kappa_\eta\le d$, then the linear interpolation paths drawn from the two leaf couplings do not intersect.
\end{proposition}

\begin{proof}
Without loss of generality, assume $\ell$ lies in the left subtree $\eta^-$ and $\ell'$ in the right subtree $\eta^+$. Since $\kappa_\eta\le d$, the CQAT tree performs a data-coordinate split at $\eta$. By the tree construction, for any $(x_1,c_1)\in R_\ell$ and $(x_1',c_1')\in R_{\ell'}$,
\[
x_1^{(\kappa_\eta)}\le s_\eta < (x_1')^{(\kappa_\eta)}.
\]
Moreover, since $\eta$ is a data-coordinate split node, the source-side region is simultaneously partitioned: for any $x_0\in\widetilde{R}_\ell$ and $x_0'\in\widetilde{R}_{\ell'}$,
\[
x_0^{(\kappa_\eta)}\le\widetilde{s}_\eta < (x_0')^{(\kappa_\eta)}.
\]
Hence, for all $t\in[0,1]$,
\[
(x_t')^{(\kappa_\eta)}-x_t^{(\kappa_\eta)}
=
(1-t)\bigl[(x_0')^{(\kappa_\eta)}-x_0^{(\kappa_\eta)}\bigr]
+
t\bigl[(x_1')^{(\kappa_\eta)}-x_1^{(\kappa_\eta)}\bigr]
>0,
\]
which implies $x_t'\neq x_t$ for all $t\in[0,1]$.
\end{proof}

\begin{proposition}[Path separation property of CQAT]
\label{prop:cqat_path_separation}
Let $x_t=(1-t)x_0+tx_1$ and $x_t'=(1-t)x_0'+tx_1'$, where $x_1,x_1'$ denote the data-coordinate components of joint samples. For any $t\in(0,1)$,
\[
\EE_{\pi_{\mathrm{CQAT}}^{\otimes 2}}\!\left[\|x_t-x_t'\|_2^2\right]
-
\EE_{\pi_{\mathrm{ind}}^{\otimes 2}}\!\left[\|x_t-x_t'\|_2^2\right]
=
4t(1-t)\Delta_{\mathrm{CQAT}},
\]
where $\Delta_{\mathrm{CQAT}}$ is the node-level separation induced by the Conditional QAT tree structure:
\begin{equation}
\label{eq:cqat_path_separation}
\begin{aligned}
\Delta_{\mathrm{CQAT}}
&:=
\sum_{\eta\in\mathcal{I}_c:\,\kappa_\eta\le d}
\alpha_\eta q_\eta(1-q_\eta)
\Bigl(\EE_{\gamma_{\eta^+}}\!\bigl[X^{(\kappa_\eta)}\bigr]-\EE_{\gamma_{\eta^-}}\!\bigl[X^{(\kappa_\eta)}\bigr]\Bigr)
\\[-2pt]
&\qquad\times
\Bigl(\EE_{\widehat{\mu}_{\mathcal{D}_c,\eta^+}}\!\bigl[X^{(\kappa_\eta)}\bigr]-\EE_{\widehat{\mu}_{\mathcal{D}_c,\eta^-}}\!\bigl[X^{(\kappa_\eta)}\bigr]\Bigr).
\end{aligned}
\end{equation}
Here $\mathcal{I}_c$ is the set of internal nodes of the Conditional QAT tree; $\alpha_\eta:=\widehat{\mu}_{\mathcal{D}_c}(R_\eta)$ is the joint data mass at node $\eta$; and $q_\eta:=\widehat{\mu}_{\mathcal{D}_c,\eta}(R_{\eta^-})$ is the conditional mass of the left child.
\end{proposition}

\begin{proof}
The telescoping recursion~\eqref{eq:proof_cross_final} in the proof of Theorem~\ref{thm:path-separation} relies only on the left-right child decomposition at each internal node and the product structure of the leaf couplings. Replacing the QAT tree with the Conditional QAT tree and $\widehat{\mu}_{\mathcal{D}}$ with the joint empirical distribution $\widehat{\mu}_{\mathcal{D}_c}$, the same argument gives
\[
\EE_{\pi_{\mathrm{CQAT}}^{\otimes 2}}
\!\left[
(X_0-X_0')^\top(X_1-X_1')
\right]
=
2\sum_{\eta\in\mathcal{I}_c}
\frac{\alpha_{\eta^-}\alpha_{\eta^+}}{\alpha_\eta}
\Delta_\eta,
\]
where $\alpha_\eta:=\widehat{\mu}_{\mathcal{D}_c}(R_\eta)$ and $\Delta_\eta$ is the local cross-term contribution at node $\eta$.

It remains to show that condition-coordinate split nodes contribute zero. If $\kappa_\eta>d$, the CQAT tree does not partition the source-side region at $\eta$, so both child nodes inherit the same source region: $\gamma_{\eta^-}=\gamma_{\eta^+}$. Consequently,
\[
\EE_{\gamma_{\eta^+}}[X]=\EE_{\gamma_{\eta^-}}[X],
\]
and the source-side mean difference vanishes, giving $\Delta_\eta=0$. Therefore, only data-coordinate split nodes ($\kappa_\eta\le d$) contribute, and the sum reduces to
\[
\EE_{\pi_{\mathrm{CQAT}}^{\otimes 2}}
\!\left[
(X_0-X_0')^\top(X_1-X_1')
\right]
=
2\Delta_{\mathrm{CQAT}}.
\]
Combining with the linear interpolation expansion~\eqref{eq:proof_path_expand} and the fact that the cross-term vanishes under the independent coupling $\pi_{\mathrm{ind}}$, we obtain
\[
\EE_{\pi_{\mathrm{CQAT}}^{\otimes 2}}\!\left[\|X_t-X_t'\|_2^2\right]
-
\EE_{\pi_{\mathrm{ind}}^{\otimes 2}}\!\left[\|X_t-X_t'\|_2^2\right]
=
4t(1-t)\Delta_{\mathrm{CQAT}}.
\]
Under non-degenerate splits, if at least one data-coordinate split node contributes positively, then $\Delta_{\mathrm{CQAT}}>0$. This completes the proof.
\end{proof}

\section{Algorithms}
\label{sec:alg}

\subsection{QAT Coupling Construction and Sampling}
\label{subsec:qatalg}

This section presents pseudocode for QAT coupling construction and training-pair sampling. QAT uses $\gamma=\mathcal{N}(0,I_d)$ as the source distribution and maintains source-side coordinate bounds $(a_\ell,b_\ell)\in(\mathbb{R}\cup\{\pm\infty\})^d\times(\mathbb{R}\cup\{\pm\infty\})^d$ at each leaf node. Throughout, $\Phi$ and $\Phi^{-1}$ denote the standard-normal CDF and its inverse, respectively.

\begin{algorithm}[ht]
\caption{QAT Coupling Construction}
\label{alg:qat_construction}
\begin{algorithmic}[1]
\Require Dataset $\mathcal{D}=\{x_i\}_{i=1}^N\subseteq\mathbb{R}^d$
\State $\mathcal{S}_r\gets\{1,\ldots,N\}$, $a_r\gets(-\infty,\ldots,-\infty)$, $b_r\gets(+\infty,\ldots,+\infty)$
\State $\mathcal{Q}\gets\{r\}$, $\mathcal{L}\gets\varnothing$
\While{$\mathcal{Q}\neq\varnothing$}
    \State Pop one node $\eta$ from $\mathcal{Q}$
    \If{$|\mathcal{S}_\eta|\le 1$}
        \State Add $(\mathcal{S}_\eta,\,a_\eta,\,b_\eta)$ to $\mathcal{L}$
    \Else
        \State $\kappa_\eta\gets\arg\max_{j\in[d]}\operatorname{Var}_{\{x_i:i\in\mathcal{S}_\eta\}}\!(x^{(j)})$,\quad $s_\eta\gets\frac{1}{|\mathcal{S}_\eta|}\sum_{i\in\mathcal{S}_\eta}x_i^{(\kappa_\eta)}$
        \State $\mathcal{S}_{\eta^-}\gets\{i\in\mathcal{S}_\eta:x_i^{(\kappa_\eta)}\le s_\eta\}$,\quad $\mathcal{S}_{\eta^+}\gets\mathcal{S}_\eta\setminus\mathcal{S}_{\eta^-}$,\quad $q_\eta\gets|\mathcal{S}_{\eta^-}|/|\mathcal{S}_\eta|$
        \State $\widetilde{s}_\eta\gets\Phi^{-1}\!\bigl(\Phi(a_\eta^{(\kappa_\eta)})+q_\eta\cdot\bigl(\Phi(b_\eta^{(\kappa_\eta)})-\Phi(a_\eta^{(\kappa_\eta)})\bigr)\bigr)$
        \State $(a_{\eta^-},b_{\eta^-})\gets(a_\eta,b_\eta)$;\quad $(a_{\eta^+},b_{\eta^+})\gets(a_\eta,b_\eta)$
        \State $b_{\eta^-}^{(\kappa_\eta)}\gets\widetilde{s}_\eta$;\quad $a_{\eta^+}^{(\kappa_\eta)}\gets\widetilde{s}_\eta$
        \State Push $\eta^-$ and $\eta^+$ into $\mathcal{Q}$
    \EndIf
\EndWhile
\State \Return $\mathcal{L}=\{(\mathcal{S}_\ell,\,a_\ell,\,b_\ell):\ell\in\mathcal{L}\}$
\end{algorithmic}
\end{algorithm}

Because the construction terminates when $|\mathcal{S}_\ell|\le 1$, each non-empty leaf holds exactly one training sample. Every leaf therefore carries mass $1/N$, and sampling a leaf is equivalent to drawing a data index uniformly at random.

\begin{algorithm}[ht]
\caption{QAT Coupling Training-Pair Sampling}
\label{alg:qat_sampling}
\begin{algorithmic}[1]
\Require Leaf list $\mathcal{L}=\{(\mathcal{S}_\ell,\,a_\ell,\,b_\ell):\ell\in\mathcal{L}\}$
\State Sample $\ell\sim\operatorname{Unif}(\mathcal{L})$;\quad let $\mathcal{S}_\ell=\{i\}$ and set $x\gets x_i$
\For{$j=1,\ldots,d$}
    \State Sample $u_j\sim\operatorname{Unif}\bigl(\Phi(a_\ell^{(j)}),\,\Phi(b_\ell^{(j)})\bigr)$;\quad set $\epsilon^{(j)}\gets\Phi^{-1}(u_j)$
\EndFor
\State \Return $(\epsilon,\,x)$
\end{algorithmic}
\end{algorithm}

In flow matching training, the sampled pair $(\epsilon,x)$ serves as $(x_0,x_1)$.

\paragraph{Complexity analysis.}
Assuming the tree is approximately balanced, it has $\mathcal{O}(\log N)$ levels. At each level, computing coordinate means and variances over all $N$ samples costs $\mathcal{O}(Nd)$, giving a total construction cost of $\mathcal{O}(Nd\log N)$. Sampling a single training pair requires generating $d$ independent truncated-Gaussian variates via the inverse-CDF method, each in $O(1)$, for a per-sample cost of $\mathcal{O}(d)$. Epoch-level sampling therefore costs $\mathcal{O}(Nd)$, matching standard Gaussian sampling in asymptotic complexity.

\paragraph{Vectorized implementation.}
In practice, the source-side bounds for each sample's corresponding leaf are precomputed and stacked into two matrices $\A,\B\in(\mathbb{R}\cup\{\pm\infty\})^{N\times d}$ indexed by data order. Given a batch of indices $\mathcal{I}=\{i_b\}_{b=1}^B$, one retrieves $\A_{\mathcal{I},:}$ and $\B_{\mathcal{I},:}$ directly and draws $U_{bj}\sim\operatorname{Unif}\bigl(\Phi(\A_{i_b,j}),\Phi(\B_{i_b,j})\bigr)$ in parallel, then sets $\epsilon_{bj}=\Phi^{-1}(U_{bj})$. The entire batch sampling step is therefore fully vectorizable.

\subsection{CQAT Coupling Construction and Sampling}
\label{subsec:cqat_alg}

In CQAT, the condition variable is embedded as a Euclidean vector and concatenated with the data feature to form $y_i=(x_i,c_i)\in\mathbb{R}^{d+m}$; the algorithm then builds a tree directly on the joint sample set $\mathcal{D}_c=\{y_i\}_{i=1}^N$. To prevent late-stage splits driven purely by condition coordinates — which partition the data index set without updating the source-side bounds — we restrict the candidate split coordinates to $[d]$ whenever the node sample count falls below a threshold $n_{\mathrm{data}}$. Algorithm~\ref{alg:cqat_construction} gives the full construction.

\begin{algorithm}[ht]
\caption{CQAT Coupling Construction}
\label{alg:cqat_construction}
\begin{algorithmic}[1]
\Require Joint dataset $\mathcal{D}_c=\{y_i\}_{i=1}^N\subseteq\mathbb{R}^{d+m}$, data-split threshold $n_{\mathrm{data}}$
\State $\mathcal{S}_r\gets\{1,\ldots,N\}$, $a_r\gets(-\infty,\ldots,-\infty)$, $b_r\gets(+\infty,\ldots,+\infty)$
\State $\mathcal{Q}\gets\{r\}$, $\mathcal{L}\gets\varnothing$
\While{$\mathcal{Q}\neq\varnothing$}
    \State Pop one node $\eta$ from $\mathcal{Q}$
    \If{$|\mathcal{S}_\eta|\le 1$}
        \State Add $(\mathcal{S}_\eta,\,a_\eta,\,b_\eta)$ to $\mathcal{L}$
    \Else
        \State $\mathcal{J}_\eta\gets[d]$ if $|\mathcal{S}_\eta|\le n_{\mathrm{data}}$, else $\mathcal{J}_\eta\gets[d+m]$
        \State $\kappa_\eta\gets\arg\max_{j\in\mathcal{J}_\eta}\operatorname{Var}_{\{y_i:i\in\mathcal{S}_\eta\}}\!(y^{(j)})$,\quad $s_\eta\gets\frac{1}{|\mathcal{S}_\eta|}\sum_{i\in\mathcal{S}_\eta}y_i^{(\kappa_\eta)}$
        \State $\mathcal{S}_{\eta^-}\gets\{i\in\mathcal{S}_\eta:y_i^{(\kappa_\eta)}\le s_\eta\}$,\quad $\mathcal{S}_{\eta^+}\gets\mathcal{S}_\eta\setminus\mathcal{S}_{\eta^-}$
        \State $(a_{\eta^-},b_{\eta^-})\gets(a_\eta,b_\eta)$;\quad $(a_{\eta^+},b_{\eta^+})\gets(a_\eta,b_\eta)$
        \If{$\kappa_\eta\le d$}
            \State $q_\eta\gets|\mathcal{S}_{\eta^-}|/|\mathcal{S}_\eta|$
            \State $\widetilde{s}_\eta\gets\Phi^{-1}\!\bigl(\Phi(a_\eta^{(\kappa_\eta)})+q_\eta\cdot\bigl(\Phi(b_\eta^{(\kappa_\eta)})-\Phi(a_\eta^{(\kappa_\eta)})\bigr)\bigr)$
            \State $b_{\eta^-}^{(\kappa_\eta)}\gets\widetilde{s}_\eta$;\quad $a_{\eta^+}^{(\kappa_\eta)}\gets\widetilde{s}_\eta$
        \EndIf
        \State Push $\eta^-$ and $\eta^+$ into $\mathcal{Q}$
    \EndIf
\EndWhile
\State \Return $\mathcal{L}=\{(\mathcal{S}_\ell,\,a_\ell,\,b_\ell):\ell\in\mathcal{L}\}$
\end{algorithmic}
\end{algorithm}

\paragraph{Scale balancing.}
Since CQAT builds on data and condition coordinates jointly, scale imbalance between $x_i$ and $c_i$ may bias split selection toward high-variance condition dimensions. We therefore build the tree on $(X,\,wC)$. For discrete class labels, we set $w=\infty$, which is formally equivalent to first partitioning all samples by class and then constructing an independent QAT within each class subtree — matching the approach of C2OT~\citep{cheng2025curse}. For continuous conditions, we adopt the variance-normalized weight
\begin{equation}
\label{eq:conditional_weight}
    w=\left(\frac{\sum_j\operatorname{Var}(X^{(j)})}{\sum_k\operatorname{Var}(C^{(k)})}\right)^{1/2},
\end{equation}
so that the rescaled condition variable has aggregate variance comparable to that of the data features.

\paragraph{Complexity analysis.}
Building the tree over $d+m$ coordinates gives a construction cost of $\mathcal{O}(N(d+m)\log N)$ under a balanced-tree assumption. Condition-coordinate splits ($\kappa_\eta>d$) partition the data index set without touching the source-side bounds; data-coordinate splits ($\kappa_\eta\le d$) perform quantile alignment exactly as in QAT. At training time, Algorithm~\ref{alg:qat_sampling} applies directly: one returns the data sample $x_i$ together with its paired condition variable $c_i$. Since source sampling remains confined to $d$-dimensional Gaussian space, the per-sample cost is $\mathcal{O}(d)$, unchanged from QAT.

\subsection{High-Dimensional Image Data Processing}
\label{subsec:hd_alg}

For high-dimensional image data, we do not build the QAT directly in raw pixel space. Instead, we first apply an orthogonal transform to concentrate the principal variation into the leading coordinates. Crucially, only a pure rotation is applied — no mean subtraction — so that if the source is $\mathcal{N}(0,I_d)$, it remains standard Gaussian under both the forward and inverse transforms.

\paragraph{Patch-Hadamard transform.}
Let the input image have spatial resolution $H\times W$ with $C$ channels, flattened to $x_i\in\mathbb{R}^d$ ($d=HWC$). For patch size $p$, the image is divided into $d'=(H/p)(W/p)C$ non-overlapping patches, each of spatial dimension $r=p^2$. A normalized Hadamard matrix $H_{\mathrm{patch}}\in\mathbb{R}^{r\times r}$~\citep{horadam2012hadamard} is applied independently to each patch. Because the normalized Hadamard matrix is orthogonal with constant direction $r^{-1/2}\mathbf{1}_r$, extracting the coefficient along this direction is equivalent to patch-level mean pooling. We denote the resulting mean-pooled representation as $\bar{x}_i\in\mathbb{R}^{d'}$.

Let $H_p\in\mathbb{R}^{d\times d}$ denote the global patch-wise Hadamard transform induced by patch size $p$. Since each per-patch transform is orthogonal, $H_p$ is itself orthogonal. We order the output coordinates so that the first $d'$ entries of $H_p x_i$ correspond to the constant-direction coefficients across all patches. In practice, if only the low-dimensional tree representation is needed, $\bar{x}_i$ can be obtained directly by patch-level mean pooling, with PCA directions estimated on this $d'$-dimensional representation.

\paragraph{PCA rotation and low-dimensional QAT.}
Patch-Hadamard pooling avoids the cost of PCA in the full $d$-dimensional pixel space. We estimate the PCA rotation matrix $\mathbf{P}\in\mathbb{R}^{d'\times d'}$ from $\{\bar{x}_i\}_{i=1}^N$ without mean shift. Setting $z_i=\mathbf{P}\bar{x}_i\in\mathbb{R}^{d'}$, we use only the leading $k$ coordinates $z_i^{(1:k)}$ to construct the QAT coupling. Algorithm~\ref{alg:hd_qat_construction} summarizes the full construction pipeline.

\begin{algorithm}[ht]
\caption{High-Dimensional QAT Construction}
\label{alg:hd_qat_construction}
\begin{algorithmic}[1]
\Require Image dataset $\mathcal{D}=\{x_i\}_{i=1}^N$, patch size $p$, reduced dimension $k$
\State $d'\gets(H/p)(W/p)C$
\State $\bar{x}_i\gets(H_p x_i)^{(1:d')}$, \quad for $i=1,\ldots,N$
\State Estimate PCA rotation $\mathbf{P}\in\mathbb{R}^{d'\times d'}$ from $\{\bar{x}_i\}_{i=1}^N$ without mean shift
\State $\widetilde{z}_i\gets(\mathbf{P}\bar{x}_i)^{(1:k)}$, \quad for $i=1,\ldots,N$
\State Run Algorithm~\ref{alg:qat_construction} on $\{\widetilde{z}_i\}_{i=1}^N$
\State \Return $\mathcal{L}$, $\mathbf{P}$, $H_p$
\end{algorithmic}
\end{algorithm}

\paragraph{Inverse-transform sampling.}
The low-dimensional QAT constructs a coupling over the leading $k$ PCA-pooled coordinates, whose source marginal is $\mathcal{N}(0,I_k)$. To recover a source sample in the original $d$-dimensional pixel space, we fill the remaining coordinates with independent standard Gaussian noise: first the $d'-k$ residual PCA-pooled coordinates, then the high-frequency Hadamard coefficients for each patch (all per-patch dimensions except the constant direction). Since $\mathbf{P}$ and $H_p$ are both orthogonal and all unfilled coordinates are drawn as independent standard Gaussians, the resulting source sample has marginal exactly $\mathcal{N}(0,I_d)$.

\begin{algorithm}[ht]
\caption{High-Dimensional QAT Sampling}
\label{alg:hd_qat_sampling}
\begin{algorithmic}[1]
\Require $\mathcal{D}=\{x_i\}_{i=1}^N$, leaf list $\mathcal{L}=\{(\mathcal{S}_\ell,\,a_\ell,\,b_\ell):\ell\in\mathcal{L}\}$, $\mathbf{P}$, $H_p$, $d'$, $k$
\State Sample one leaf $\ell\sim\operatorname{Unif}(\mathcal{L})$;\quad let $\mathcal{S}_\ell=\{i\}$ and set $x\gets x_i$
\For{$j=1,\ldots,k$}
    \State Sample $u_j\sim\operatorname{Unif}\bigl(\Phi(a_\ell^{(j)}),\,\Phi(b_\ell^{(j)})\bigr)$;\quad set $\epsilon^{(j)}\gets\Phi^{-1}(u_j)$
\EndFor
\State Sample $\epsilon^{(k+1:d')}\sim\mathcal{N}(0,I_{d'-k})$
\State $\epsilon_{\mathrm{pool}}\gets\bigl(\epsilon^{(1:k)},\,\epsilon^{(k+1:d')}\bigr)$
\State $\bar{\epsilon}\gets\mathbf{P}^\top\epsilon_{\mathrm{pool}}$
\State Sample $\epsilon_{\mathrm{high}}\sim\mathcal{N}(0,I_{d-d'})$
\State $\epsilon_H\gets\bigl(\bar{\epsilon},\,\epsilon_{\mathrm{high}}\bigr)$
\State $x_0\gets H_p^\top\epsilon_H$
\State \Return $(x_0,\,x)$
\end{algorithmic}
\end{algorithm}

\paragraph{Complexity analysis.}
Let $d=HWC$, $r=p^2$, and $d'=(H/p)(W/p)C$. Computing patch-level mean pooling costs $\mathcal{O}(Nd)$; the full patch-wise Hadamard transform can be applied via the Fast Walsh-Hadamard Transform~\citep{horadam2012hadamard} in $\mathcal{O}(Nd\log r)$. PCA is estimated on the $d'$-dimensional pooled representation, and the low-dimensional QAT is built in $k$-dimensional space, giving tree construction cost $\mathcal{O}(Nk\log N)$. At sampling time, the inverse PCA rotation costs $\mathcal{O}(d'^2)$ per sample and the inverse patch-wise Hadamard transform costs $\mathcal{O}(d\log r)$, giving a total per-sample cost of $\mathcal{O}(d+d'^2+d\log r)$; batch sampling is fully vectorizable.

\section{Implementation Details}
\label{sec:imple}

This section details the per-experiment training configurations. Unless stated otherwise, all methods share the same network architecture, training configuration, and sampling protocol, differing only in the coupling strategy used to form training pairs: I-FM uses independent coupling, OT-FM uses mini-batch OT coupling, and QAT-FM uses the QAT or Conditional QAT coupling proposed in this work.

\subsection{Two-Dimensional Experiment}
\label{subsec:impl_2d}

\paragraph{Data.}
Following~\citet{lipman2022flow,pooladian2023multisample}, we consider transport from a two-dimensional standard Gaussian source distribution to a checkerboard target distribution. Training samples are drawn independently from the source and target; the flow paths learned by I-FM, OT-FM, SD-FM, and QAT-FM are then compared under the same network architecture and optimization configuration.

\paragraph{Network and training.}
Following~\citet{lipman2024flow}, the velocity field is parameterized by a five-layer MLP with hidden width $512$ and Swish activations~\citep{ramachandran2017searching}. All methods are trained for $20$K iterations with batch size $256$ and Adam~\citep{kingma2015adam} at learning rate $10^{-3}$.
For QAT-FM, $50$K samples are drawn from the checkerboard target to build the QAT tree, and training pairs are sampled from the induced coupling. For OT-FM, exact EMD is solved within each mini-batch using POT~\citep{flamary2021pot}. For SD-FM, a discrete reference measure is built from $50$K target samples and the corresponding regularized problem is solved with entropy parameter $0.05$; we optimize with Adam, as the SGD-based strategy of~\citet{genevay2016stochastic,mousavi2026flow} proved unstable and slow to converge in this setting.

\begin{table}[t]
\centering
\caption{\footnotesize Hyperparameter settings for image-space experiments.}
\label{tab:impl_image_hparams}
\small
\begin{tabular}{lccc}
\toprule
 & CIFAR-10 & ImageNet-$32$ & CelebA-$64$ \\
\midrule
Channels & 192 & 256 & 192 \\
Depth & 2 & 3 & 3 \\
Channel multiplier & 1,2,2,2 & 1,2,2,2 & 1,2,3,4 \\
Heads & 4 & 4 & 4 \\
Head channels & 64 & 64 & 64 \\
Attention resolution & 16 & 4 & 8 \\
Dropout & 0.0 & 0.0 & 0.0 \\
Batch size / GPU & 256 & 128 & 64 \\
GPUs & 1 & 4 & 4 \\
Effective batch size & 256 & 512 & 256 \\
Iterations & $100$k & $200$k & $100$k \\
Learning rate & 2e-4 & 1e-4 & 1e-4 \\
LR scheduler & Constant & Linear decay & Constant \\
Warmup steps & 5k & 20k & 10k \\
\bottomrule
\end{tabular}
\end{table}

\subsection{CIFAR-10}

We conduct unconditional image generation on CIFAR-10~\citep{krizhevsky2009learning}, using the U-Net architecture of~\citet{tong2024improving}. Main network and training hyperparameters are listed in Table~\ref{tab:impl_image_hparams}. We train with Adam, a linear warmup to $2\times10^{-4}$ over $5$k steps (constant thereafter), gradient norm clipping at $1.0$, and an EMA decay of $0.9999$.

For QAT-FM, the QAT coupling is built from the $50{,}000$ CIFAR-10 training images together with their horizontally flipped augmentations, giving $n=100{,}000$ tree samples. Images are first processed by patch-level mean pooling with patch size $p=4$; PCA rotation is then estimated on the pooled representation, and the leading $k=32$ coordinates are retained for tree construction. The tree splits on the maximum-variance coordinate at the empirical mean threshold and recurses until each leaf contains at most one sample. At evaluation, $50$K images are generated, and FID is computed using \texttt{clean-fid}~\citep{parmar2022aliased}, following the protocol of~\citet{tong2024improving}.

\subsection{ImageNet-32}

We evaluate class-conditional generation on ImageNet-1k~\citep{deng2009imagenet} at $32\times32$ resolution. All images are resized to $32\times32$ with random horizontal flipping, and class labels are provided as condition inputs to the velocity-field network. We follow the class-conditional U-Net architecture of~\citet{tong2024improving}; main hyperparameters are listed in Table~\ref{tab:impl_image_hparams}. Optimizer, gradient clipping, and EMA settings match the CIFAR-10 setup. The learning rate warms up linearly to $1\times10^{-4}$ over $20$k steps and then decays linearly to $3\times10^{-5}$.

For QAT-FM, the Conditional QAT coupling is built from ImageNet-1k training images and their horizontal-flip augmentations, yielding over $2$M tree samples. Patch-level mean pooling ($p=4$) and PCA rotation are applied as in the CIFAR-10 setup, retaining the leading $k=32$ coordinates. Since the condition is a discrete class label, training samples are first partitioned by class and an independent QAT coupling is constructed within each class subset.
For C2OT~\citep{cheng2025curse}, we follow its oversampling strategy with ratio $50$, yielding a local OT batch size of $128\times50=6400$ per GPU. Following the C2OT setting for class-conditional generation, we use its conditional weighting scheme and set the condition weight to $10^8$.

At evaluation, $50$K image samples are generated to estimate Inception feature statistics of the generated distribution; reference statistics are computed from all ${\sim}1.28$M ImageNet-1k training images, following~\citet{dhariwal2021diffusion,ma2024sit}.

\subsection{CelebA-64}

\paragraph{Data.}
We evaluate text-conditional generation on CelebA~\citep{liu2015deep} at $64\times64$ resolution. Since CelebA does not provide text annotations by default, we use the natural-language descriptions from CelebA-Dialog~\citep{jiang2021talk} as text conditions.

\paragraph{Network and Training.}
We use the U-Net architecture with main hyperparameters in Table~\ref{tab:impl_image_hparams}. Following~\citet{cheng2025curse}, scale-shift normalization is applied in the U-Net residual blocks, and text embeddings are projected to the conditioning dimension via a linear layer before being fed to the network. Optimizer, gradient clipping, and EMA settings match the CIFAR-10 experiment; the learning rate warms up linearly to $1\times10^{-4}$ over $10$k steps and remains constant thereafter.

For QAT-FM, a Conditional QAT coupling is constructed in the joint space of low-dimensional image representations and text embeddings. The image-side representation is obtained via patch-level mean pooling ($p=4$) and PCA rotation, retaining the leading $k=32$ coordinates. Text conditions are encoded by a \texttt{DFN5B-CLIP-ViT-H-14} text encoder~\citep{fang2024data}; the resulting embeddings are projected to $32$ dimensions via PCA and used as condition-side features in the joint tree. To prevent late-stage splits driven purely by condition coordinates, candidate split dimensions are restricted to data coordinates when a node's sample count falls below $n_{\mathrm{data}}=128$. 
For C2OT~\citep{cheng2025curse}, we follow its oversampling strategy with ratio $50$, a local OT batch size of $64\times50=3200$ per GPU, which substantially increases the coupling cost; see Appendix~\ref{subsec:runtime_benchmark}. For a fair conditional comparison, we use the Euclidean distance between text embeddings as the conditional distance and apply the weighting scheme in~\eqref{eq:conditional_weight}.

\paragraph{Evaluation.}
At evaluation, $50$K images are generated and assessed along two axes: image fidelity via FID against training-set reference statistics~\citep{ning2023input,wang2023patch}, and text-image alignment via CLIP score computed with the \texttt{siglip2-base-patch16-224} checkpoint of SigLIP-2~\citep{tschannen2025siglip2}.

\subsection{ImageNet-256}

We conduct class-conditional latent-space generation on ImageNet-$256\times256$~\citep{deng2009imagenet}. Images are encoded by VA-VAE~\citep{yao2025reconstruction} into latent representations, and a flow-based LightningDiT~\citep{yao2025reconstruction} is trained in the resulting latent space with class labels as condition inputs. All network architecture, optimizer, learning-rate schedule, EMA, precision, and sampling configurations follow the official implementation of~\citet{yao2025reconstruction}; due to computational constraints, we adopt the recommended lightweight training setting of $64$ epochs.

For QAT-FM, the Conditional QAT coupling is built directly over VA-VAE latent representations. PCA projection is applied, retaining the leading $k=256$ coordinates for tree construction; patch-level mean pooling is omitted since VA-VAE already maps images to a compact latent space. The tree splits on the maximum-variance coordinate at the empirical mean and recurses until each leaf contains at most one sample. Classifier-free guidance is applied during inference with a scale of $10$.

At evaluation, we follow the protocol of~\citet{yao2025reconstruction} and compute FID against the ImageNet-$256\times256$ reference statistics from guided-diffusion~\citep{dhariwal2021diffusion}. We additionally report IS, sFID, Precision, and Recall.

\section{Extended Experimental Results}
\label{sec:exp_add}

This section provides supplementary results comprising a runtime comparison of coupling strategies and additional visualization of generated samples. Appendix~\ref{subsec:runtime_benchmark} benchmarks the wall-clock time of each coupling approach at both construction and sampling phases; the subsequent four subsections present additional generated samples from CIFAR-10, ImageNet-$32\times32$, CelebA-$64\times64$, and ImageNet-$256\times256$.

\subsection{Runtime Comparison}
\label{subsec:runtime_benchmark}

We compare the wall-clock time of different coupling strategies at both construction and sampling phases. All experiments fix the feature dimension at $d=32$ and vary the dataset size $n$. To reflect the repeated pairing demand of FM training, we report cumulative times over $100$ training epochs. For OT-based methods, source points are resampled from a standard Gaussian and the coupling is recomputed from scratch at every epoch; by contrast, the QAT and SD-FM couplings are constructed once offline and reused throughout training. Full OT is solved with the \texttt{EMD} solver from POT~\citep{flamary2021pot}; mini-batch OT (OTBatch) uses \texttt{ot.dist\_batch} to compute batch pairwise distances, fully exploiting GPU parallelism. Entropy regularization with $\varepsilon=0.05$ is applied for both SD-FM and batch OT. SD-FM optimizes a semi-discrete coupling with SGD~\citep{mousavi2026flow}, running until convergence or $100$ optimization epochs.

\begin{figure}[ht]
    \centering
    \includegraphics[width=.75\linewidth]{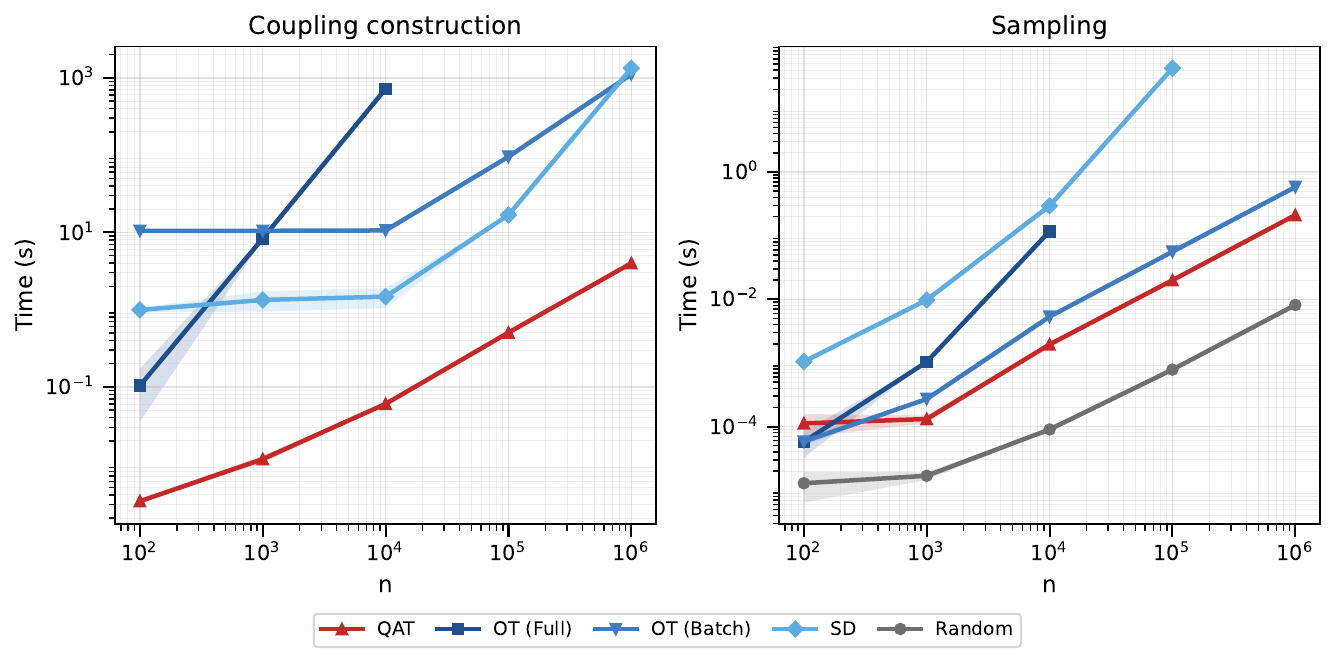}
    \caption{\footnotesize
    Runtime comparison of different coupling strategies.
    Left: total coupling construction time over $100$ epochs.
    Right: total sampling time over $100$ epochs.}
    \label{fig:runtime_benchmark}
\end{figure}

Figure~\ref{fig:runtime_benchmark} reports cumulative construction and sampling times over $100$ epochs. Full OT becomes rapidly infeasible as $n$ grows; batch OT reduces the per-epoch cost but must recompute the coupling every epoch, keeping its total time high. SD-FM requires only a single construction pass, but the semi-discrete optimization itself is expensive. By contrast, QAT constructs the tree and records the Gaussian bounds at each leaf in a single offline pass, achieving substantially lower construction time at all scales. At sampling time, QAT draws source points by querying leaf bounds and sampling from truncated Gaussians, placing its per-sample cost close to that of independent (random) sampling. OT-based methods must sample from the coupling's joint distribution, and SD-FM requires a nearest-neighbor lookup over all training points at each draw---both become prohibitive at large $n$. Overall, QAT achieves structured coupling at a sampling efficiency comparable to independent coupling.

\newpage
\subsection{CIFAR-10}

\begin{figure}[ht]
    \centering
    \includegraphics[width=\linewidth]{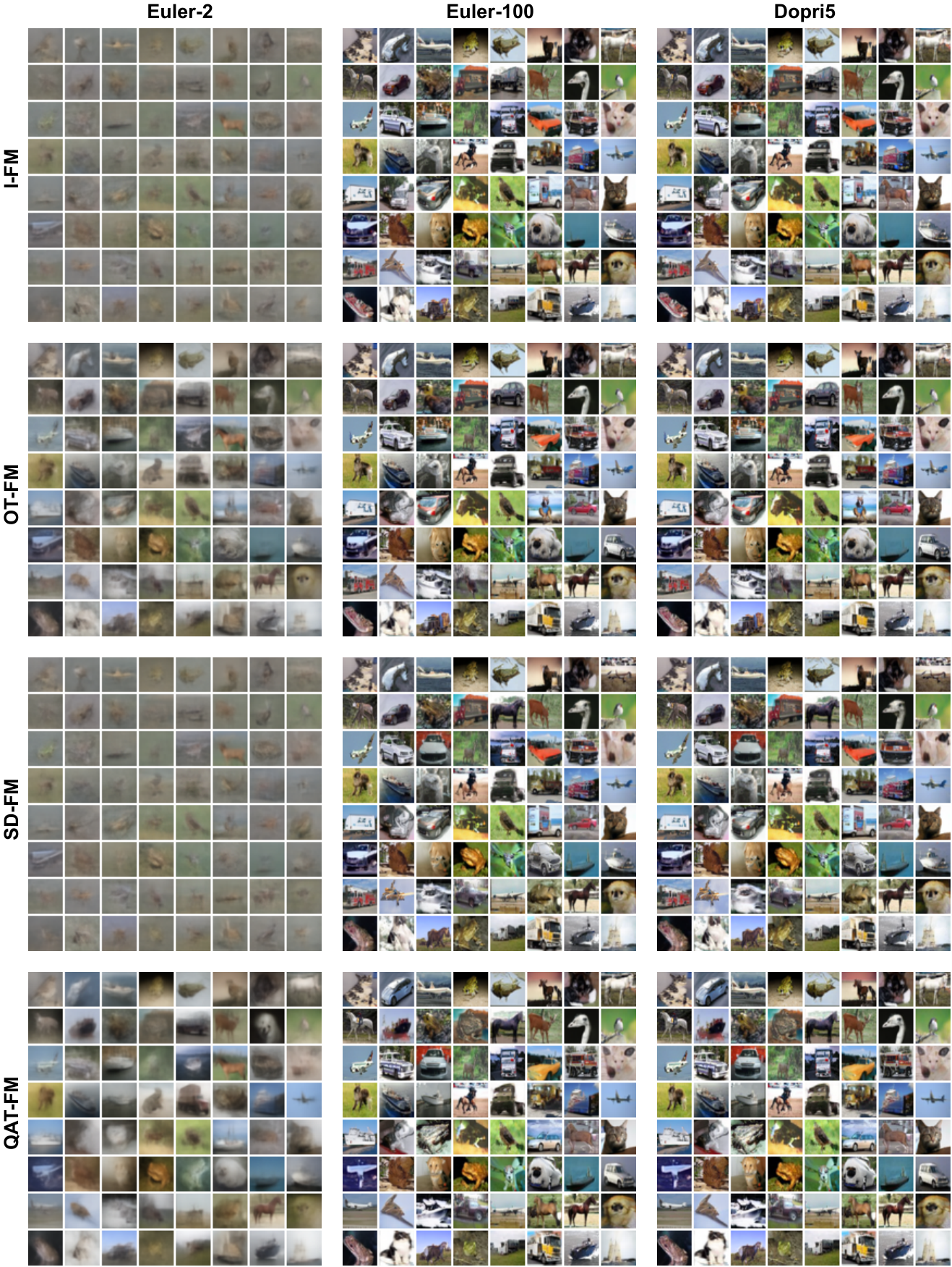}
    \caption{\footnotesize
    Additional generated samples for unconditional CIFAR-10 generation. Rows correspond to I-FM, OT-FM, SD-FM, and QAT-FM; columns compare three sampling configurations: \texttt{Euler-2}, \texttt{Euler-100}, and \texttt{dopri5}. All images are randomly generated without cherry-picking.
    }
    \label{fig:cifar10_fvis}
\end{figure}

\newpage
\subsection{ImageNet-32}

\begin{figure}[ht]
    \centering
    \includegraphics[width=\linewidth]{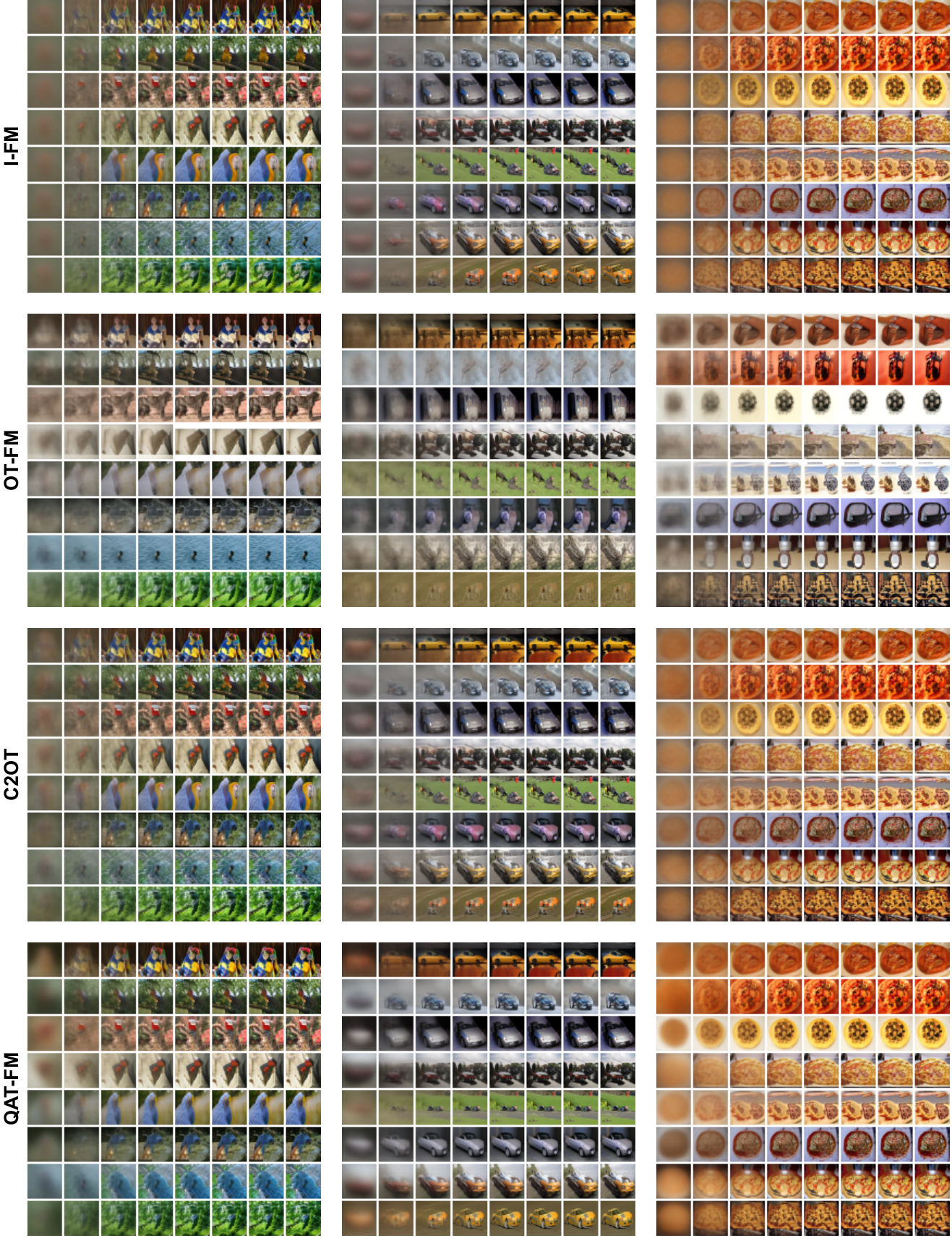}
    \caption{\footnotesize
    Additional generated samples for class-conditional ImageNet-$32\times32$ generation. The four row blocks correspond to I-FM, OT-FM, C2OT, and QAT-FM, respectively. The three column blocks correspond to macaw, sports car, and pizza class conditions. Within each class block, the eight columns show samples generated with \texttt{Euler-1}, \texttt{Euler-2}, \texttt{Euler-5}, \texttt{Euler-10}, \texttt{Euler-20}, \texttt{Euler-50}, \texttt{Euler-100}, and \texttt{Dopri5}, respectively. All images are randomly generated without cherry-picking.
    }
    \label{fig:img32_fvis}
\end{figure}

\newpage
\subsection{CelebA-64}

\begin{figure}[ht]
    \centering
    \includegraphics[width=.93\linewidth]{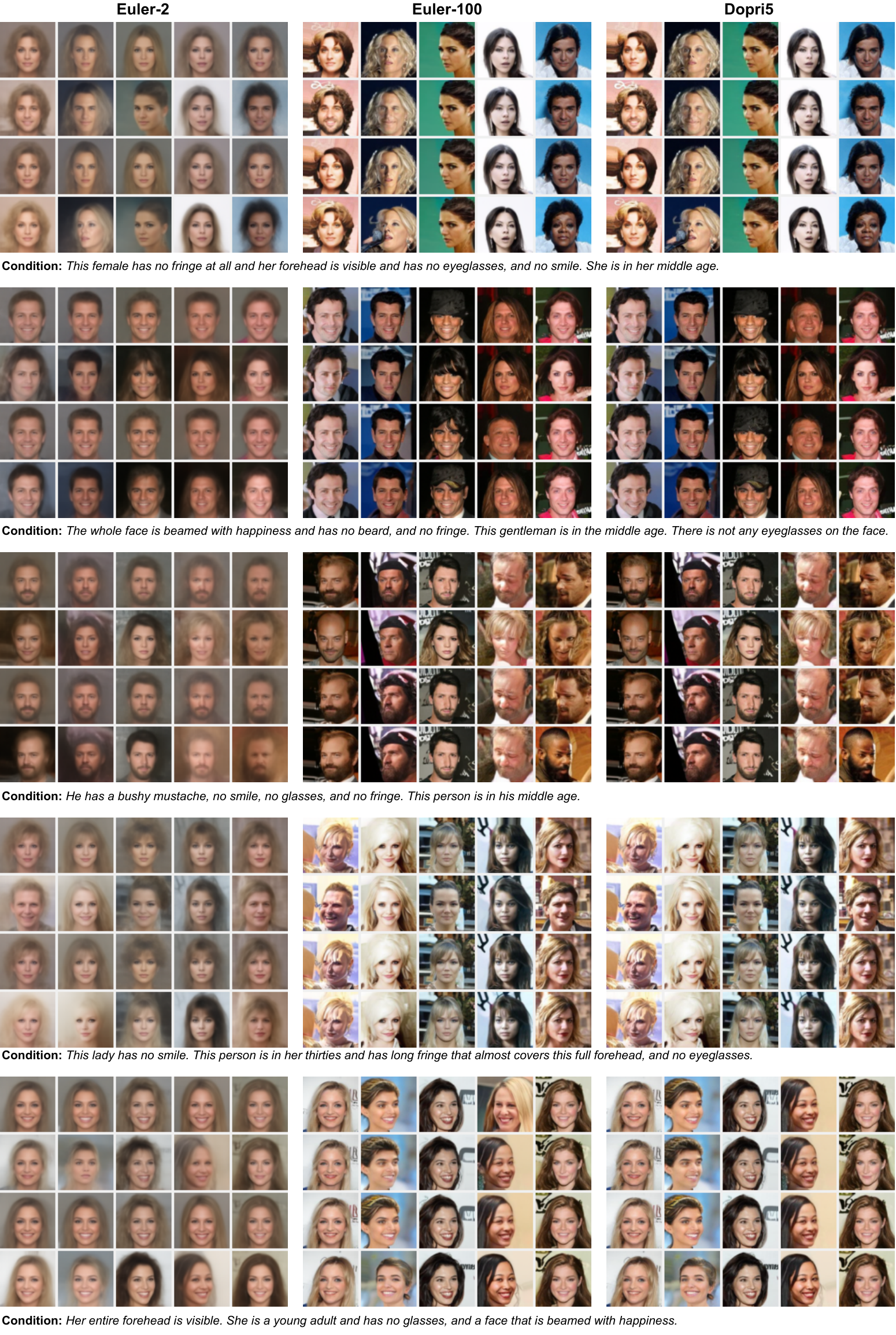}
    \caption{\footnotesize
    Additional generated samples for text-conditional CelebA-$64\times64$ generation. Each horizontal block corresponds to one CelebA-Dialog text condition; the three columns show results under \texttt{Euler-2}, \texttt{Euler-100}, and \texttt{Dopri5}; within each block, the four rows correspond to I-FM, OT-FM, C2OT and QAT-FM. All images are randomly generated without cherry-picking.
    }
    \label{fig:celeba_fvis}
\end{figure}

\newpage
\subsection{ImageNet-256}

\begin{figure}[ht]
    \centering
    \includegraphics[width=.94\linewidth]{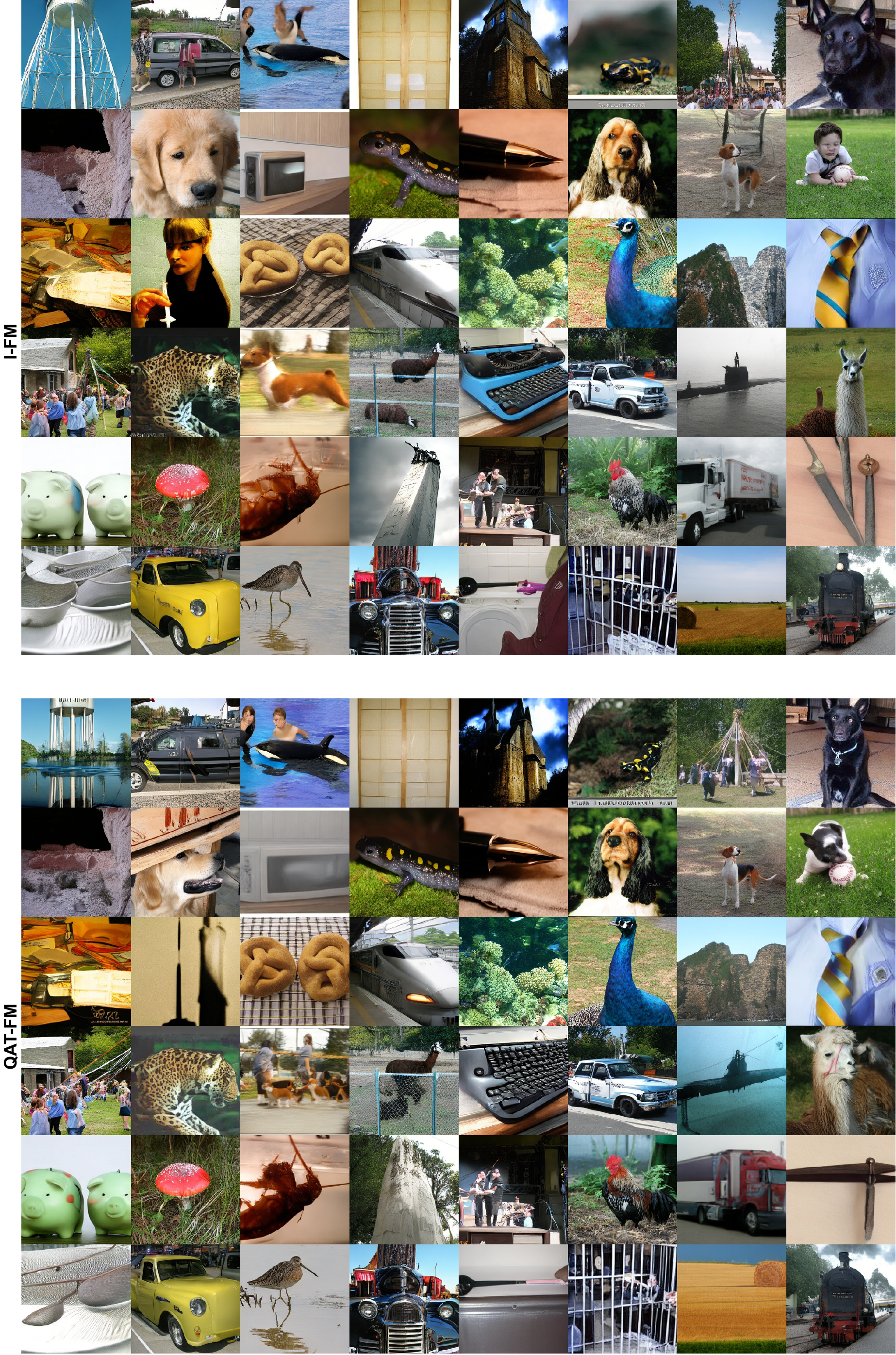}
    \caption{\footnotesize
    Additional generated samples for class-conditional latent-space generation on ImageNet-$256\times256$. The top and bottom groups show samples from I-FM and QAT-FM, respectively, generated under the same random seed and class conditions. All images are randomly generated without cherry-picking.
    }
    \label{fig:img256_fvis}
\end{figure}

\end{document}